\documentclass[11pt]{article}
\usepackage{natbib}
\usepackage{ifthen}
\usepackage{mdwlist}
\usepackage{amsmath,amssymb,amsfonts,amsthm}
\usepackage{bm}
\usepackage[colorlinks=true,linkcolor=Pink,citecolor=Pink,urlcolor=Pink]{hyperref}
\usepackage{enumitem}
\usepackage{graphicx}
\usepackage{xspace}
\usepackage{verbatim}
\usepackage[margin=1in]{geometry}
\usepackage{color}
\usepackage{thm-restate}
\usepackage{latexsym}
\usepackage{epsfig}
\usepackage[colorinlistoftodos,textsize=scriptsize]{todonotes}
\usepackage{fontawesome5}

\usepackage{algorithm}
\usepackage{algorithmic}
\usepackage[utf8]{inputenc} 
\usepackage[T1]{fontenc}    
\usepackage{hyperref}       
\usepackage{url}            
\usepackage{booktabs}       
\usepackage{amsfonts}       
\usepackage{nicefrac}       
\usepackage{microtype}      
\usepackage{xcolor}         
\usepackage{amsmath}
\usepackage{amsthm}
\usepackage{graphicx}
\usepackage{multirow}
\usepackage{titletoc}

\definecolor{Pink}{RGB}{240, 82, 156}
\def\withcolors{1}
\def\withnotes{1}

\renewcommand{\epsilon}{\ve}
\def\ve{\varepsilon}

\newcommand{\E}{\mbox{\bf E}}

\newcommand{\pr}[2][]{\mathrm{Pr}\ifthenelse{\not\equal{}{#1}}{_{#1}}{}\!\left[#2\right]}
\newcommand{\R}{\mathbb{R}}

\newtheorem{theorem}{Theorem}

\newtheorem{corollary}[theorem]{Corollary}

\numberwithin{theorem}{section} 
\numberwithin{nontheorem}{section} 
\numberwithin{proposition}{section} 
\numberwithin{observation}{section} 
\numberwithin{remark}{section} 
\numberwithin{fact}{section} 
\numberwithin{lemma}{section} 
\numberwithin{claim}{section} 
\numberwithin{corollary}{section} 
\numberwithin{case}{section} 
\numberwithin{dfn}{section} 
\numberwithin{definition}{section} 
\numberwithin{question}{section} 
\numberwithin{openquestion}{section} 
\numberwithin{res}{section}

\newcommand{\cL}{\mathcal{L}}

\newcommand{\method}{DriftXpress}

\ifnum\withcolors=1
  \newcommand{\gcolor}[1]{{\color{red}#1}} 
\else

  \newcommand{\gcolor}[1]{{#1}}
\fi

\ifnum\withnotes=1
  \newcommand{\gnote}[1]{\par\gcolor{\textbf{G: }\sf #1}} 
  \newcommand{\gfootnote}[1]{\footnote{{\bf \gcolor{Gautam}}: {#1}}}

\else
  \newcommand{\gnote}[1]{}
  \newcommand{\gfootnote}[1]{}

\fi
\newcommand{\ignore}[1]{\leavevmode\unskip} 

\title{DriftXpress: Faster Drifting Models via Projected RKHS Fields\thanks{EC, GK, and SM are listed in alphabetical order.}}

\author {
Ali Falahati\thanks{Cheriton School of Computer Science, University of Waterloo and Vector Institute. {\tt afalahat@uwaterloo.ca}.} \and Elliot Creager\thanks{Department of Electrical and Computer Engineering, University of Waterloo and Vector Institute. {\tt creager@uwaterloo.ca.}} \and 
Gautam Kamath\thanks{Cheriton School of Computer Science, University of Waterloo and Vector Institute. {\tt g@csail.mit.edu}. Supported by a Canada CIFAR AI Chair, an NSERC Discovery Grant, and an Ontario Early Researcher Award.} \and Shubhankar Mohapatra\thanks{A*STAR Centre for Frontier AI Research, Singapore. {\tt shubhankar\_mohapatra@a-star.edu.sg.} Supported by an NRF postdoctoral award.} }

\begin{document}
\maketitle

\begin{abstract}
Drifting Models have emerged as a new paradigm for one-step generative modeling, achieving strong image quality without iterative inference. The premise is to replace the iterative denoising process in diffusion models with a single evaluation of a generator. However, this creates a different trade-off: drifting reduces inference cost by moving much of the computation into training. We introduce \textbf{DriftXpress}, an accelerated formulation of drifting models based on projected RKHS fields. DriftXpress approximates the drifting kernel in a low-rank feature space. This preserves the attraction-repulsion structure of the original drifting field while reducing the cost of field evaluation. 
Across image-generation benchmarks, DriftXpress achieves comparable FID to standard drifting while reducing wall-clock training cost. These results show that the training-inference trade-off of drifting models can be pushed further without giving up their one-step inference advantage.

\vspace{0.5em}
\begin{center}
\noindent\faGithub\  Code: \texttt{\href{https://github.com/Mortrest/DriftXpress}{github.com/Mortrest/DriftXpress}}
\end{center}
\end{abstract}

\section{Introduction}

One-step generative models are appealing because they reduce sampling to a single forward pass. However, achieving this with high fidelity has proven difficult.
Diffusion and flow-based generators~\citep{sohl2015deep,ho2020denoising,song2020score,lipman2022flow,liu2022flow,albergo2023stochastic} obtain strong sample quality by distributing generation across many iterative refinement steps, each individually simple.
Collapsing these steps into one, whether by distillation~\citep{salimans2022progressive,luo2023comprehensive,yin2024one,zhou2024score} or by direct single-step training~\citep{song2023consistency,frans2024one,geng2025mean,boffi2024flow}, typically requires either a pretrained multi-step teacher or approximations of the underlying ODE or SDE trajectories.
Generative Adversarial Networks (GANs)~\citep{goodfellow2014generative} are natively one-step, but rely on adversarial optimization and inherit its well-known instabilities~\citep{mescheder2018which, arjovsky2017towards, salimans2016improved}.

Drifting Models~\citep{deng2026drifting} provide an alternative route to one-step generation by learning an attraction--repulsion field in feature space: generated samples are pulled toward the data distribution and pushed away from the current model distribution. Their premise is to replace diffusion-style iterative inference with a single generator evaluation. While drifting models are effective in practice, matching or surpassing strong one-step and multi-step baselines, they introduce a different trade-off: inference is cheap, but training is expensive because the model must repeatedly estimate a data-dependent vector field. This paper asks whether that trade-off can be pushed further:
\begin{quote}
\emph{``Can we reduce the training cost of drifting models while preserving their one-step inference advantage?''}
\end{quote}

Addressing this challenge is critical, as once inference is reduced to a single step, training becomes the primary computational bottleneck. We observe that in standard drifting, this bottleneck arises from estimating the field via kernel evaluations between generated and training samples. Since the field depends on these evaluations, we focus on reducing the cost of field estimation by replacing exact kernel interactions with projected kernel evaluations while preserving the field's form.

The reproducing kernel Hilbert space (RKHS) perspective provides the foundation for kernel methods. RKHSs are widely used because positive-definite kernels enable nonlinear estimators to be studied through linear geometry in function space, which underlies a broad class of methods, including support vector machines and kernel principal component analysis~\citep{cortes1995support,scholkopf1998nonlinear}. The Nystr\"om method is a standard tool for scaling such kernel-based algorithms~\cite{drineas2005nystrom}. Rather than working with the full kernel matrix or operator, it approximates the kernel interaction structure using a subset of landmark columns, yielding a low-rank surrogate that preserves the kernel formulation while reducing computational cost~\citep{williams2001using}.

Motivated by this view, we introduce \method, a projected RKHS formulation of drifting models (Figure~\ref{fig:overview}). Instead of recomputing exact attraction against the training support at every step, \method{} approximates these kernel interactions using a Nystr\"om feature map built from landmark points selected before training. This turns repeated training-support interactions into fast query-landmark evaluations and cached summary multiplications, while preserving the attraction--repulsion form of the drifting field. Empirically, this yields two sources of acceleration: cheaper field evaluation per step and faster wall-clock convergence, since projected attraction uses a summary of the full positive support rather than only mini-batch interactions.

Our contributions are threefold. First, we recast drifting fields in the RKHS induced by the drifting kernel, enabling low-rank landmark-based evaluation. Second, we introduce \method, which replaces repeated exact attraction with cached summaries of training data and composes these summaries across shards for larger datasets. Third, we show that \method{} preserves comparable sample quality while substantially reducing training cost.

\begin{figure}[t]
    \centering
    \vspace{-0.75em}
    \includegraphics[width=\linewidth]{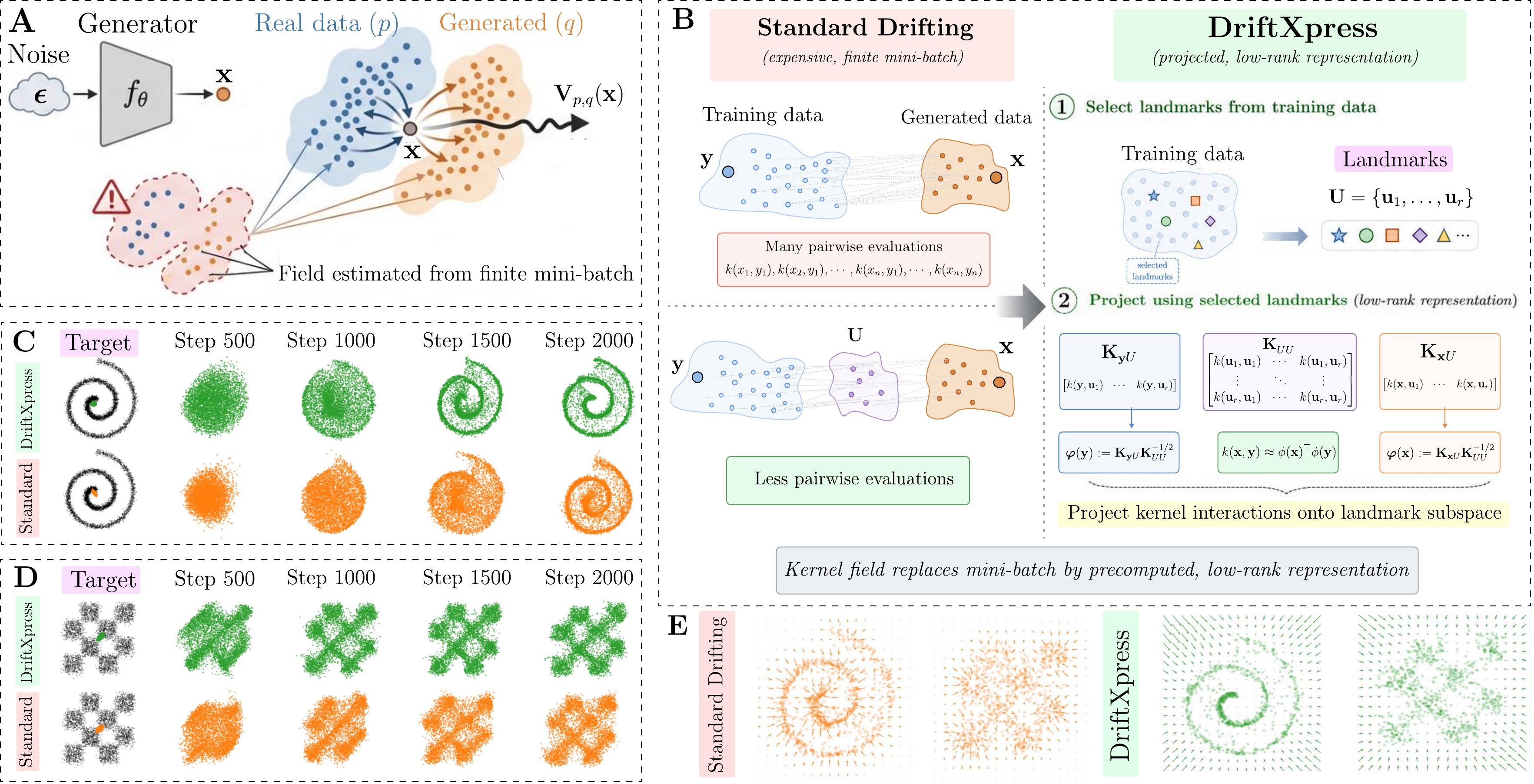}
\caption{\textbf{Overview of \method{}.}
(A) Drifting trains a generator by estimating an attraction--repulsion field from finite mini-batches.
(B) \method{} replaces repeated exact attraction with a Nystr\"om landmark projection.
(C--D) Swissroll and checkerboard training trajectories show faster convergence to structured targets than standard drifting.
(E) \method{} yields smoother vector fields that better follow the target geometry.}
    \label{fig:overview}
\vspace{-1.25em}
\end{figure}

\section{Background: Drifting Models}
\label{sec:background}

Drifting Models~\citep{deng2026drifting} provide a one-step perspective on generative modeling by constructing a transport rule from noise to the data distribution. We briefly review the Drifting Model framework.
 
\paragraph{Pushforward at training time.}
Let $f_\theta: \R^C \to \R^D$ be a generator network, $p$ be the data distribution, and $p_\epsilon$ be a noise distribution on $\R^C$. For a noise input $\boldsymbol{\epsilon}\sim p_\epsilon$, the generator output is
$\mathbf{x}=f_\theta(\boldsymbol{\epsilon})\in\R^D$.
The output $\mathbf{x}$ follows the pushforward distribution
$q=(f_\theta)_{\#}p_\epsilon$, that is, the distribution obtained by mapping samples from $p_\epsilon$ through $f_\theta$.
As training iterates, the sequence of models $\{f_{\theta_i}\}$ induces a sequence of model distributions
$q_i=(f_{\theta_i})_{\#}p_\epsilon$ that converge towards $p$.

\paragraph{Drifting field.}
The evolution of samples at iteration $i$ is governed by
$\mathbf{V}_{p,q_i}: \R^D \to \R^D$:
\begin{equation}
  \mathbf{x}_{i+1}
  =
  \mathbf{x}_i + \mathbf{V}_{p,q_i}(\mathbf{x}_i).
  \label{eq:drift_update}
\end{equation}
For a generic model distribution $q$, the field satisfies the anti-symmetry property
$\mathbf{V}_{p,q} = -\mathbf{V}_{q,p}$, which guarantees that
$\mathbf{V}_{p,q}(\mathbf{x}) = 0$ for all $\mathbf{x}$ when $p=q$.

\paragraph{Attraction-repulsion instantiation.}
The drifting field decomposes into attraction toward data and repulsion from generated samples:
\begin{equation}
  \mathbf{V}_{p,q}(\mathbf{x})=
  \mathbf{V}^+_p(\mathbf{x})
  -
  \mathbf{V}^-_q(\mathbf{x}),
  \label{eq:drifting_field}
\end{equation}
where $\mathbf{y}^+ \sim p$ denotes a data sample and $\mathbf{y}^- \sim q$ denotes a generated sample. The two components are
\begin{equation}
  \mathbf{V}^+_p(\mathbf{x})
  =
  \frac{1}{Z_p(\mathbf{x})}
  \E_{\mathbf{y}^+ \sim p}
  \big[
  k(\mathbf{x}, \mathbf{y}^+)(\mathbf{y}^+ - \mathbf{x})
  \big],
  \qquad
  \mathbf{V}^-_q(\mathbf{x})
  =
  \frac{1}{Z_q(\mathbf{x})}
  \E_{\mathbf{y}^- \sim q}
  \big[
  k(\mathbf{x}, \mathbf{y}^-)(\mathbf{y}^- - \mathbf{x})
  \big],
\end{equation}
with normalization factors
\begin{equation}
  Z_p(\mathbf{x})
  =
  \E_{\mathbf{y}^+ \sim p}
  \big[
  k(\mathbf{x}, \mathbf{y}^+)
  \big],
  \qquad
  Z_q(\mathbf{x})
  =
  \E_{\mathbf{y}^- \sim q}
  \big[
  k(\mathbf{x}, \mathbf{y}^-)
  \big],
\end{equation} where the default kernel $k$ is the Laplace kernel. 

\paragraph{Training objective.}
The training loss drives $f_\theta$ toward the drifted target via a stop-gradient:
\begin{equation}
  \cL=\E_{\boldsymbol{\epsilon}}\Big[\big\|f_\theta(\boldsymbol{\epsilon})-
  \texttt{stopgrad}
  \big(f_\theta(\boldsymbol{\epsilon})+\mathbf{V}_{p,q}(f_\theta(\boldsymbol{\epsilon}))\big)\big\|^2\Big].
  \label{eq:training_loss}
\end{equation}

In practice, the field is evaluated in the feature space of a frozen encoder, and the expectations in
$\mathbf{V}_{p,q}$ are estimated empirically during training. The main computational bottleneck is therefore not the loss itself, but the repeated evaluation of $\mathbf{V}_{p,q}$ through normalized kernel interactions between generated samples and training samples. \method\ targets this bottleneck by replacing exact kernel evaluations with projected evaluations that can reuse precomputed statistics from the training data.

\section{DriftXpress: RKHS Projected Drifting Fields}
\label{sec:method}

We introduce \method, a projected RKHS formulation of drifting fields. The field is written as a difference of normalized kernel-weighted averages, and scalability comes from replacing exact kernel sections with projections onto a landmark-induced subspace. We implement this projection with the Nystr\"om approximation, which expresses the kernel through evaluations against a small landmark set and yields the compact summaries used by \method{}.

\subsection{RKHS Projection via Nystr\"om Landmarks}

Let $\mathcal{H}_k$ denote the RKHS associated with the kernel $k$, and let
\[
\mathcal{H}_{\mathbf{U}} := \mathrm{span}\{k(\mathbf{u}_1,\cdot),\dots,k(\mathbf{u}_r,\cdot)\} \subset \mathcal{H}_k
\]
be the landmark-induced subspace associated with a set of landmarks $\mathbf{U}=\{\mathbf{u}_1,\dots,\mathbf{u}_r\}\subset\R^D$. Landmarks are selected from the training samples. We sample landmarks uniformly at random within each class. The resulting landmarks are fixed before training and define the subspace $\mathcal H_{\mathbf U}$. We study alternative landmark-selection rules in Section~\ref{sec:landmark_selection}.

DriftXpress approximates each kernel section $k(\mathbf{x},\cdot)$ by its projection onto $\mathcal{H}_{\mathbf{U}}$.
In practice, this projection is implemented through the Nystr\"om feature map.
Define the landmark kernel matrix $\mathbf{K}_{UU}\in\R^{r\times r}$ with entries
\[
[\mathbf{K}_{UU}]_{ab}=k(\mathbf{u}_a,\mathbf{u}_b).
\]
For any point $\mathbf{z}\in\R^D$, define
\[
\mathbf{K}_{\mathbf{z}U}
=
\begin{bmatrix}
k(\mathbf{z},\mathbf{u}_1),\dots,k(\mathbf{z},\mathbf{u}_r)
\end{bmatrix}
\in \R^{1\times r}.
\]
The Nystr\"om feature map is\footnote{For numerical stability, $\mathbf{K}_{UU}^{-1/2}$ is replaced by $(\mathbf{K}_{UU}+\lambda \mathbf{I})^{-1/2}$}
\begin{equation}
  \boldsymbol{\varphi}(\mathbf{z})
  :=
  \left(\mathbf{K}_{\mathbf{z}U}\mathbf{K}_{UU}^{-1/2}\right)^\top
  \in \R^r.
  \label{eq:nystrom_features}
\end{equation}
For any two points $\mathbf{z},\mathbf{z}'\in\R^D$, the Nystr\"om features approximate their kernel similarity through the landmark-based low-rank kernel
\begin{equation}
  k_{\mathbf{U}}(\mathbf{z},\mathbf{z}')
  :=
  \boldsymbol{\varphi}(\mathbf{z})^\top \boldsymbol{\varphi}(\mathbf{z}')
  \approx
  k(\mathbf{z},\mathbf{z}').
  \label{eq:nystrom_kernel}
\end{equation}
Thus, instead of evaluating the full kernel $k$, DriftXpress works with the low-rank kernel $k_{\mathbf{U}}$ associated with the landmark subspace $\mathcal{H}_{\mathbf{U}}$.

\subsection{RKHS Projected Approximation of the Drifting Field}

The attractive component of the drifting field can be written in terms of a normalized kernel-weighted mean map:
\begin{equation}
  \mu_p(\mathbf{x})
  :=
  \frac{\sum_{j=1}^{N^+} k(\mathbf{x},\mathbf{y}_j^+)\,\mathbf{y}_j^+}
       {\sum_{j=1}^{N^+} k(\mathbf{x},\mathbf{y}_j^+)},
  \qquad
  \mathbf{V}_p^+(\mathbf{x}) = \mu_p(\mathbf{x}) - \mathbf{x}.
\end{equation}
DriftXpress replaces $k$ by the projected kernel $k_{\mathbf{U}}$, yielding
\begin{equation}
  \mu_p^{\mathbf{U}}(\mathbf{x})
  :=
  \frac{\sum_{j=1}^{N^+} \boldsymbol{\varphi}(\mathbf{x})^\top \boldsymbol{\varphi}(\mathbf{y}_j^+)\,\mathbf{y}_j^+}
       {\sum_{j=1}^{N^+} \boldsymbol{\varphi}(\mathbf{x})^\top \boldsymbol{\varphi}(\mathbf{y}_j^+)},
  \qquad
  \mathbf{V}_{p,\mathbf{U}}^+(\mathbf{x}) = \mu_p^{\mathbf{U}}(\mathbf{x}) - \mathbf{x}.
\end{equation}
Since $\boldsymbol{\varphi}(\mathbf{x})^\top$ does not depend on $j$, we can factor it out and precompute the following \emph{summaries}:
\begin{equation}
  \mathbf{A}_p := \sum_{j=1}^{N^+} \boldsymbol{\varphi}(\mathbf{y}_j^+)\,\mathbf{y}_j^{+\top}
  \;\in\; \R^{r\times D},
  \qquad
  \mathbf{b}_p := \sum_{j=1}^{N^+} \boldsymbol{\varphi}(\mathbf{y}_j^+)
  \;\in\; \R^r.
  \label{eq:summaries_p}
\end{equation}
This gives the compact expression
\begin{equation}
  \mu_p^{\mathbf{U}}(\mathbf{x})
  =
  \frac{\mathbf{A}_p^\top \boldsymbol{\varphi}(\mathbf{x})}
       {\boldsymbol{\varphi}(\mathbf{x})^\top \mathbf{b}_p+\varepsilon}.
  \label{eq:projected_positive_mean_compact}
\end{equation}
The denominator is stabilized in implementation by adding a positive offset $\varepsilon > 0$ before division. For the repulsive component, \method{} keeps the original exact kernel-weighted mean:
\begin{equation}
  \mu_q(\mathbf{x})
  :=
  \frac{\sum_{j=1}^{N^-} k(\mathbf{x}, \mathbf{y}_j^-)\mathbf{y}_j^-}
       {\sum_{j=1}^{N^-} k(\mathbf{x}, \mathbf{y}_j^-)+\varepsilon} ,
  \qquad
  \mathbf{V}_{q}^-(\mathbf{x}) = \mu_q(\mathbf{x}) - \mathbf{x}.
  \label{eq:exact_repulsive_mean}
\end{equation}
Here $\{\mathbf{y}_j^-\}_{j=1}^{N^-}$ denotes the generated samples from the current model distribution. In practice, this set is the generated batch at the current training step, with self-interactions removed when the query point is also part of the repulsive support. Therefore, the \method{} drifting field uses projected attraction and exact repulsion:
\begin{equation}
  \mathbf{V}_{p,q}^{\mathbf{U}}(\mathbf{x})
  \;:=\;
  \mu_p^{\mathbf{U}}(\mathbf{x}) - \mu_q(\mathbf{x})
  \;=\;
  \underbrace{\frac{\mathbf{A}_p^\top \boldsymbol{\varphi}(\mathbf{x})}
                    {\boldsymbol{\varphi}(\mathbf{x})^\top \mathbf{b}_p + \varepsilon}}_{\mu_p^{\mathbf{U}}(\mathbf{x})}
  -
  \underbrace{\frac{\sum_{j=1}^{N^-} k(\mathbf{x}, \mathbf{y}_j^-)\mathbf{y}_j^-}
                    {\sum_{j=1}^{N^-} k(\mathbf{x}, \mathbf{y}_j^-)+ \varepsilon}}_{\mu_q(\mathbf{x})}.
  \label{eq:nystrom_V}
\end{equation}
As in the exact formulation, the $\mathbf{x}$ terms in $\mathbf{V}_{p,\mathbf{U}}^+(\mathbf{x}) = \mu_p^{\mathbf{U}}(\mathbf{x})-\mathbf{x}$ and $\mathbf{V}_{q}^-(\mathbf{x})=\mu_q(\mathbf{x})-\mathbf{x}$ cancel exactly, so the implemented field is determined by a projected attractive field and an exact repulsive field. \method{} replaces the attraction term it with reusable summaries $\mathbf{A}_p$ and $\mathbf{b}_p$. These summaries can now also be computed on the entire dataset prior to training, and, as we will see in Section~\ref{sec:experiments}, they help \method{} achieve a holistic view of the data, resulting in a much faster convergence rate. The repulsive term is computed on the current generated batch, so it cannot be precomputed before training, unlike the attractive training-data summaries. Since approximating it would not remove the main offline-reusable bottleneck, we keep this term exact.

\subsection{Compositional Summaries}

The formulation above assumes that the support can be represented by a single set of Nystr\"om summaries. At sufficiently large scale, this may become infeasible: even if the per-step batch size is small, the full support may not fit into a single summary tensor. To address this, we use the compositionality of the projected field by partitioning the support into $S$ shards and maintaining a separate Nystr\"om summary for each shard.

\noindent Let shard $s\in\{1,\dots,S\}$ have a landmark set $\mathbf U^{(s)}$ and feature map
$\boldsymbol{\varphi}^{(s)}(\mathbf x)$. We partition the positive training samples across shards.
Let $\mathcal I_p^{(s)}$ denote the indices of positive samples assigned to shard $s$. The shard-local attractive summaries are
\begin{equation}
  \mathbf A_p^{(s)}
  :=
  \sum_{j\in\mathcal I_p^{(s)}} 
  \boldsymbol{\varphi}^{(s)}(\mathbf y_j^+) \mathbf y_j^{+\top},
  \qquad
  \mathbf b_p^{(s)}
  :=
  \sum_{j\in\mathcal I_p^{(s)}} 
  \boldsymbol{\varphi}^{(s)}(\mathbf y_j^+).
\end{equation}
The projected attractive mean is obtained by summing shard contributions:
\begin{equation}
  \mu_{p,\mathrm{comp}}(\mathbf x)
  =
  \frac{\sum_{s=1}^S 
  \mathbf A_p^{(s)\top}\boldsymbol{\varphi}^{(s)}(\mathbf x)}
       {\sum_{s=1}^S 
  \boldsymbol{\varphi}^{(s)}(\mathbf x)^\top \mathbf b_p^{(s)}+\varepsilon}.
  \label{eq:compositional_positive_mean}
\end{equation}
This additive form is the key reason for the possibility of sharding. Each shard contributes only a numerator and denominator term, and these terms can be accumulated sequentially without materializing the full positive summary on the device at once. Therefore, the compositional \method{} field is
\begin{equation}
  \mathbf V_{p,q}^{\mathrm{comp}}(\mathbf x)
  :=
  \mu_{p,\mathrm{comp}}(\mathbf x)-\mu_q(\mathbf x).
  \label{eq:compositional_drift_field}
\end{equation}

A key advantage of this formulation is that both the numerator and denominator decompose additively across shards. Therefore, the field can be evaluated without materializing all summaries on the device at once: shard contributions can be accumulated sequentially, so memory scales with the largest shard rather than the full support.

\begin{theorem}[Compositionality of attractive summaries]
Let $\bar{\boldsymbol{\varphi}}(\mathbf x)
:=
\big[
\boldsymbol{\varphi}^{(1)}(\mathbf x);
\dots;
\boldsymbol{\varphi}^{(S)}(\mathbf x)
\big]$
denote the concatenation of the shard feature maps, and let
\[
\bar{\mathbf A}_p
:=
\begin{bmatrix}
\mathbf A_p^{(1)}\\
\vdots\\
\mathbf A_p^{(S)}
\end{bmatrix},
\qquad
\bar{\mathbf b}_p
:=
\begin{bmatrix}
\mathbf b_p^{(1)}\\
\vdots\\
\mathbf b_p^{(S)}
\end{bmatrix},
\]
denote the concatenations of the $S$ shard feature maps and attractive summaries, respectively. Then the compositional attractive mean satisfies
\[
\mu_{p,\mathrm{comp}}(\mathbf x) = \frac{\bar{\mathbf A}_p^\top \bar{\boldsymbol{\varphi}}(\mathbf x)}{\bar{\boldsymbol{\varphi}}(\mathbf x)^\top \bar{\mathbf b}_p +\varepsilon}.
\]
Consequently, given the same repulsive mean $\mu_q(\mathbf x)$, the compositional field $\mathbf V_{p,q}^{\mathrm{comp}}(\mathbf x) = \mu_{p,\mathrm{comp}}(\mathbf x)-\mu_q(\mathbf x)$ is identical to the field obtained using a single concatenated attractive summary.
\end{theorem}

 We next provide the approximation guarantee that connects landmark quality to field quality. The key point is that the error in the projected drifting field is controlled by how well the projected kernel approximates the exact kernel interactions. Full statements and proofs are deferred to App.~\ref{app:proofs}.

\begin{theorem}[Local distortion of the projected attractive field]
\label{thm:local_field_distortion_main}
Let $\{\mathbf y_j^+\}_{j=1}^{N^+}\subset\R^D$ satisfy
$\|\mathbf y_j^+\|\le R$, and let $k$ be a nonnegative kernel. Let
$k_{\mathbf U}$ be the projected kernel induced by landmarks $\mathbf U$.
For a query $\mathbf x$, define the kernel residual vector
\[
  \mathbf r_{\mathbf U}(\mathbf x)
  :=
  \big(
  k(\mathbf x,\mathbf y_1^+) - k_{\mathbf U}(\mathbf x,\mathbf y_1^+),
  \dots,
  k(\mathbf x,\mathbf y_{N^+}^+) - k_{\mathbf U}(\mathbf x,\mathbf y_{N^+}^+)
  \big)^\top
  \in \R^{N^+},
\]
and the exact positive normalizer $
  d_p(\mathbf x)
  :=
  \frac{1}{N^+}\sum_{j=1}^{N^+} k(\mathbf x,\mathbf y_j^+).$
Assume $d_p(\mathbf x)>0$. If $
  \|\mathbf r_{\mathbf U}(\mathbf x)\|_2
  \le
  \frac{\sqrt{N^+}\,d_p(\mathbf x)}{2},$
then
\[
  \big\|
  \mathbf V_{p,\mathbf U}^+(\mathbf x)-\mathbf V_p^+(\mathbf x)
  \big\|
  \le
  \frac{4R}{\sqrt{N^+}\,d_p(\mathbf x)}
  \|\mathbf r_{\mathbf U}(\mathbf x)\|_2.
\]
\end{theorem}
\begin{corollary}[On-support distortion]
\label{cor:on_support_field_distortion_main}
Let $\mathbf K,\mathbf K_{\mathbf U}\in\R^{N^+\times N^+}$ be the exact and projected Gram matrices on
$\{\mathbf y_j^+\}_{j=1}^{N^+}$, with entries $[\mathbf K]_{ij}=k(\mathbf y_i^+,\mathbf y_j^+)$ and 
  $[\mathbf K_{\mathbf U}]_{ij}=k_{\mathbf U}(\mathbf y_i^+,\mathbf y_j^+).$
Define $\kappa_{\min}:=\min_{1\le i\le N^+} \frac{1}{N^+}\sum_{j=1}^{N^+} k(\mathbf y_i^+,\mathbf y_j^+)$. If $\|\mathbf K-\mathbf K_{\mathbf U}\|_{2,\infty}\le\frac{\sqrt{N^+}\,\kappa_{\min}}{2},$ then
\[
  \frac{1}{N^+}\sum_{i=1}^{N^+}
  \big\|
  \mathbf V_{p,\mathbf U}^+(\mathbf y_i^+)
  -
  \mathbf V_p^+(\mathbf y_i^+)
  \big\|^2
  \le
  \frac{16R^2}{{N^+}^2\kappa_{\min}^2}
  \|\mathbf K-\mathbf K_{\mathbf U}\|_F^2.
\]
\end{corollary}

Theorem~\ref{thm:local_field_distortion_main} and
Corollary~\ref{cor:on_support_field_distortion_main} show that the error introduced
by projecting the attractive field is controlled by two quantities: the landmark
kernel residual $\|\mathbf r_{\mathbf U}(\mathbf x)\|_2$ and the local kernel
mass $d_p(\mathbf x)$. In particular, the local distortion scales as
$\|\mathbf r_{\mathbf U}(\mathbf x)\|_2/d_p(\mathbf x)$. This also quantifies when $q=p$,
\method{} evaluates to $\mu_p^{\mathbf U}(\mathbf x)-\mu_p(\mathbf x)$ rather than
zero, and the theorem bounds exactly this deviation.
Thus, the projected field is accurate when the landmark kernel approximates the exact kernel well for queries sufficiently similar to the data. The corollary also explains why landmark selection matters. On the training support, the average squared field distortion is controlled by the Gram approximation error
\(\|\mathbf K-\mathbf K_{\mathbf U}\|_F^2\). Landmark selection methods that represent high-density regions, such as global \(k\)-means and facility location, should therefore yield better field estimates than methods that optimize worst-case geometric spread. This matches the results in Table~\ref{tab:nystrom_landmark_selection_time}.



\begin{table*}[t]
\centering
\setlength{\tabcolsep}{2.5pt}
\caption{\textbf{Quality and efficiency statistics for standard drifting and DriftXpress.} Best FID is the lowest FID during training; speedup is measured relative to standard drifting on the same dataset. \method{} notably improves throughput compared to standard drifting.}
\label{tab:main_results_big}
\resizebox{\textwidth}{!}{
\begin{tabular}{llcccccc}
\toprule
\textbf{Dataset} & \textbf{Method} & \textbf{Landmarks / class} & \textbf{Best FID} $\downarrow$ & \textbf{Iter. Speed} $\uparrow$ & \textbf{Throughput} $\uparrow$ & \textbf{VRAM (GB)} $\downarrow$ & \textbf{Speedup} $\uparrow$ \\
\midrule
SVHN      & Standard Drift & - & \textbf{2.94 $\pm$ 0.21} & 0.21 it/s & 2,307 img/s & \textbf{59.4} &  \\
SVHN      & DriftXpress (Non-Sharded)     & 500 & 3.11 $\pm$ 0.12 & \textbf{1.38 it/s} & \textbf{15,402 img/s} & 73.1 & \textbf{6.68$\times$} \\
SVHN      & DriftXpress (Sharded)     & 500 & 3.15 $\pm$ 0.14 & 1.31 it/s & 14,713 img/s & 62.9 & 6.38$\times$ \\
\midrule
CIFAR10  & Standard Drift & - & 5.64 $\pm$ 0.16 & 0.18 it/s & 2,463 img/s & \textbf{71.5} &  \\
CIFAR10  & DriftXpress (Non-Sharded)     & 500 & 5.52 $\pm$ 0.06 & \textbf{1.20 it/s} & \textbf{16,342 img/s} & 74.6 & \textbf{6.63$\times$} \\
CIFAR10  & DriftXpress (Sharded)     & 500 & \textbf{5.14 $\pm$ 0.11} & 1.10 it/s & 14,985 img/s & 72.3 & 6.08$\times$ \\
\midrule
CIFAR100 & Standard Drift & - & 6.24 $\pm$ 0.11 & 0.23 it/s & 2,703 img/s & \textbf{63.4} &  \\
CIFAR100 & DriftXpress (Non-Sharded)    & 70 & 9.65 $\pm$ 0.10 & \textbf{0.87 it/s} & \textbf{10,143 img/s} & 74.5 & \textbf{3.75$\times$} \\
CIFAR100 & DriftXpress (Sharded)    & 200 & \textbf{6.15 $\pm$ 0.09} & 0.66 it/s & 7,979 img/s & 71.1 & 2.95$\times$ \\
\midrule
ImageNet & Standard Drift & - & \textbf{8.83 $\pm$ 0.13} & 0.15 it/s & 2,058 img/s & \textbf{65.7} &  \\
ImageNet & DriftXpress (Non-Sharded)    & $> 20$ & - &  - & - & OOM & - \\
ImageNet & DriftXpress (Sharded)    & 150 & 9.21 $\pm$ 0.16 & \textbf{0.40 it/s} & \textbf{5,453 img/s} & 78.1 & \textbf{2.64$\times$} \\
\bottomrule
\end{tabular}
}
\end{table*}

\begin{table*}[t]
\caption{\textbf{Landmark selection ablation on CIFAR10.} Best FID over 10k training steps and average selection time for different landmark selection methods. Most density-aware methods perform competitively, while $k$-center performs substantially worse.}
\centering
\small
\setlength{\tabcolsep}{4pt}
\resizebox{\textwidth}{!}{
\begin{tabular}{lcccccc}
\toprule
\multirow{2}{*}{Method} &
\multicolumn{3}{c}{Global} &
\multicolumn{3}{c}{Per-class} \\
\cmidrule(lr){2-4}
\cmidrule(lr){5-7}
&
Best FID $\downarrow$ (1280) &
Best FID $\downarrow$ (5120) &
Time (s) $\downarrow$ &
Best FID $\downarrow$ (128/class) &
Best FID $\downarrow$ (512/class) &
Time (s) $\downarrow$ \\
\midrule
Random &
11.80 $\pm$ 0.71 &
11.34 $\pm$ 2.24 &
\textbf{0.02} &
\textbf{11.82 $\pm$ 0.60} &
\textbf{9.60 $\pm$ 0.95} &
\textbf{0.06} \\

k-means &
\textbf{11.26 $\pm$ 0.89} &
\textbf{9.37 $\pm$ 0.37} &
7.39 &
12.91 $\pm$ 0.25 &
10.52 $\pm$ 0.13 &
2.70 \\

Facility location &
12.26 $\pm$ 1.43 &
9.49 $\pm$ 0.84 &
32.70 &
11.83 $\pm$ 1.11 &
10.57 $\pm$ 0.67 &
1.21 \\

Weighted k-center &
-- &
-- &
-- &
11.94 $\pm$ 1.09 &
10.31 $\pm$ 1.40 &
17.91 \\

k-center &
28.65 $\pm$ 2.69 &
23.78 $\pm$ 2.40 &
170.52 &
23.87 $\pm$ 1.11 &
13.91 $\pm$ 1.01 &
17.79 \\
\bottomrule
\end{tabular}
}
\label{tab:nystrom_landmark_selection_time}
\end{table*}

\section{Experiments}
\label{sec:experiments}
    \vspace{-0.75em}
\begin{figure}[t]
    \centering
    \includegraphics[width=\linewidth]{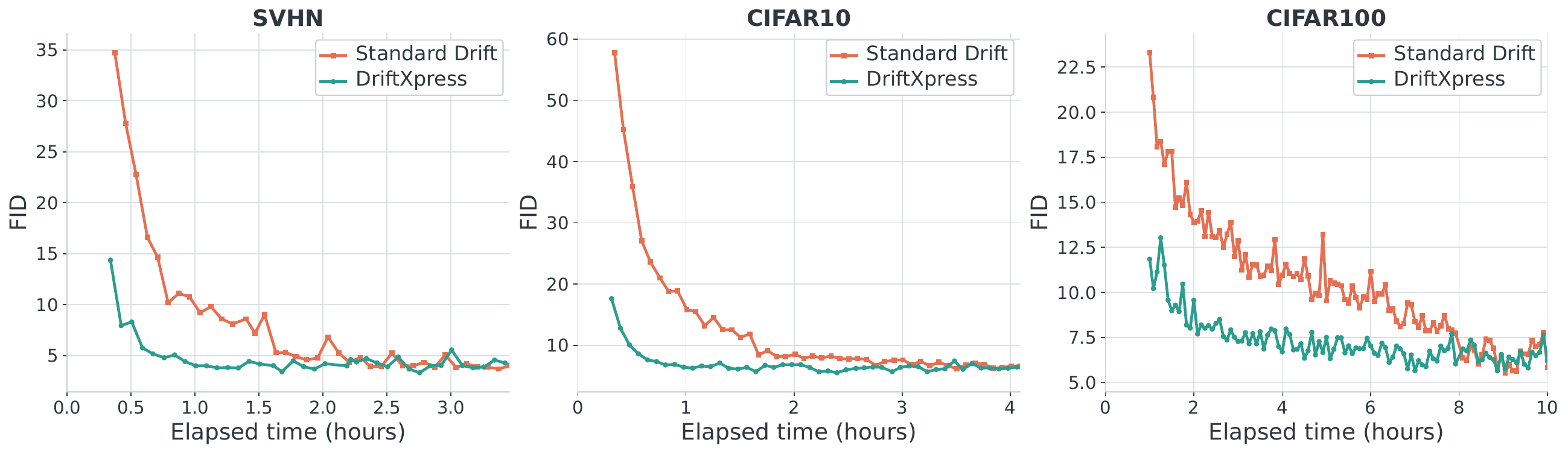}
    \caption{\textbf{FID over wall-clock training time.} FID trajectories for standard drifting and \method{} on SVHN, CIFAR10, and CIFAR100. \method{} reaches low-FID regimes faster in wall-clock time.}
\label{fig:wallclock}
\vspace{-1.25em}
\end{figure}

\begin{figure}[t]
    \centering
    \vspace{-0.75em}
    \includegraphics[width=\linewidth]{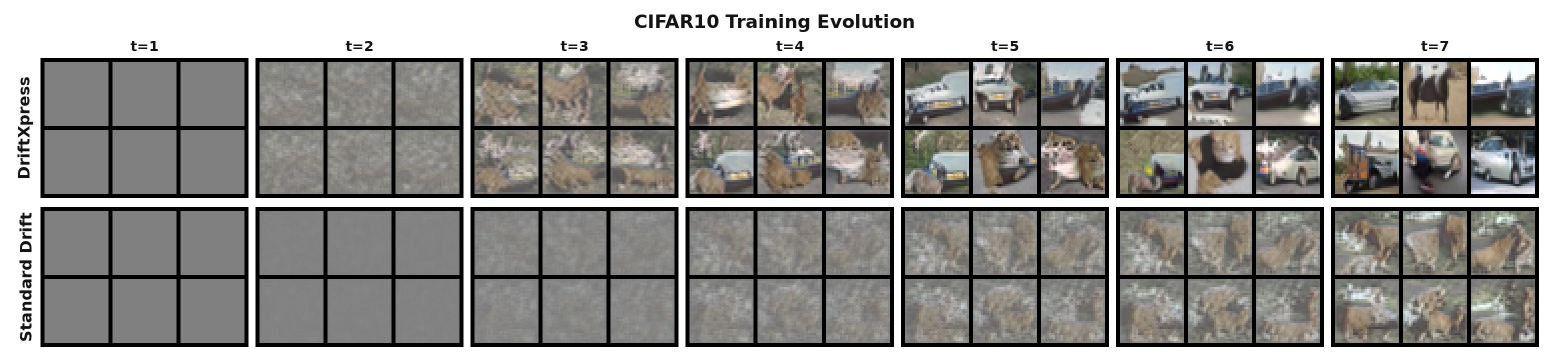}
\caption{\textbf{Early training snapshots on CIFAR10.}  Columns $t=1,\ldots,7$ denote matched training snapshot indices, with larger $t$ corresponding to later training steps. Both methods are sampled at the same training snapshots. DriftXpress reaches recognizable image structure substantially earlier, illustrating its faster training dynamics compared to Standard Drifting.}
\label{fig:speed}
\end{figure}

We evaluate DriftXpress along four axes. First, we compare its quality and efficiency against standard drifting across several image datasets. Second, we study wall-clock convergence to determine whether the projected field improves not only per-step speed, but also the rate at which FID decreases during training. Third, we isolate the effect of batch size on DriftXpress and the standard method. Finally, we ablate the landmark construction with respect to both the landmark selection strategy and the landmark ratio. The appendix provides additional experiments and implementation details. Specifically, we ablate exact versus projected repulsion (App.~\ref{app:exact-projected-repulsion}), verify that \method{} maintains one-step inference efficiency (App.~\ref{app:inference-time}), profile runtime and preprocessing gains (App.~\ref{app:cost-profile}), and evaluate the empirical and theoretical fidelity of the projected field (App.~\ref{app:field-fidelity}).


\subsection{Experimental setup} 
We compare DriftXpress against standard drifting by isolating the field estimator, keeping the generator architecture and objective identical. We evaluate sample quality primarily using Best FID, along with iteration speed, image throughput, wall-clock runtime, peak GPU memory usage, and relative speedup. 

All experiments employ a $38.3$M-parameter UNet generator~\citep{unet} and a frozen DINOv3 ViT-B/16~\citep{simeoni2025dinov3} feature encoder (required only during training). Results are reported as mean $\pm$ standard deviation across three seeds. The main benchmarks and ablations are run on 8 NVIDIA H100 GPUs and 1 NVIDIA H100 GPU, respectively. Rather than targeting computationally expensive state-of-the-art reproduction~\citep{deng2026drifting}, our controlled comparisons specifically evaluate whether RKHS-projected field estimation improves efficiency and scalability under a fixed protocol.



\subsection{Main Benchmark Results}
\label{sec:main_results}
We evaluate \method{}'s training efficiency against standard drifting under a fixed $8$-hour budget and global batch size of $12{,}000$ ($1500$/GPU). Table \ref{tab:main_results_big} details sample quality and efficiency across diverse datasets: SVHN~\citep{netzer2011reading}, CIFAR10/100~\citep{krizhevsky2009learning}, and ImageNet~\citep{deng2009imagenet} (statistics in Table \ref{tab:dataset-details}). We observe three key findings:

First, \method{} substantially increases throughput while preserving sample quality comparable to that of standard drifting. On SVHN, \method{} accelerates throughput by $6.68\times$ (from $2{,}307$ to $15{,}402$ images per second) while maintaining a comparable best FID of $3.11 \pm 0.12$ versus $2.94 \pm 0.21$. Similarly, on CIFAR10, we observe a $6.63\times$ throughput speedup, along with a slight improvement in the best FID ($5.52 \pm 0.06$ versus $5.64 \pm 0.16$). For datasets with larger class counts, the sharded variant of \method{} continues to deliver strong gains: yielding a $2.95\times$ speedup on CIFAR100 (FID $6.15 \pm 0.09$ vs.\ $6.24 \pm 0.11$) and a $2.64\times$ speedup on ImageNet (FID $9.21 \pm 0.16$ vs.\ $8.83 \pm 0.13$).

These performance gains demonstrate that modifying the training-time field estimator can drastically reduce computational overhead without altering the underlying drifting objective. While standard drifting computes exhaustive pairwise kernel interactions against sampled positive mini-batches at every step, \method{} evaluates the attractive field in a low-rank feature space using precomputed summaries of the entire training set. By replacing exact pairwise evaluations with these cached projections, \method{} efficiently exposes each generated batch to the full training support and allows for more optimization steps within the fixed wall-clock budget. We provide an online and offline cost breakdown in App.~\ref{app:cost-profile}.


Second, sharding is critical for ensuring that projected summaries remain memory-feasible at larger class counts. Because \method{} stores class-wise landmark summaries for the positive training support, unsharded memory costs scale rapidly with both class count and landmark budget. We address this by defining one shard per class ($10$ shards for SVHN/CIFAR10, $100$ for CIFAR100, and $1000$ for ImageNet). During training, the numerator and denominator of the attractive field are accumulated shard-wise. This prevents materializing a full, unsharded positive summary, thereby reducing the size of the active summary block used in each projection step while keeping all shards resident on the GPU.

Third, we observe that the sharding approach introduces a direct trade-off between memory and throughput. Without sharding, CIFAR100 runs out of memory for landmark budgets exceeding $70$ per class, and ImageNet consistently exceeds the capacity of an $80$GB GPU across all tested configurations. Sharding resolves these bottlenecks by evaluating class-wise summaries sequentially (see App.~\ref{app:shard}). However, when memory permits both variants, the unsharded approach is inherently faster as it evaluates a single projected summary. For instance, sharding on SVHN reduces peak VRAM from $73.1$GB to $62.9$GB, but drops throughput from $15{,}402$ to $14{,}713$ images/sec. We observe similar trends on CIFAR10 ($74.6$GB to $72.3$GB; $16{,}342 \to 14{,}985$ images/sec) and CIFAR100 ($74.5$GB to $71.1$GB; $10{,}143 \to 7{,}979$ images/sec). We conclude that while \method{} requires more overall memory than standard drifting to store landmarks and projected summaries, sharding provides a flexible mechanism to strictly bound this overhead.


\begin{figure}[t]
    \centering
    \includegraphics[width=\linewidth]{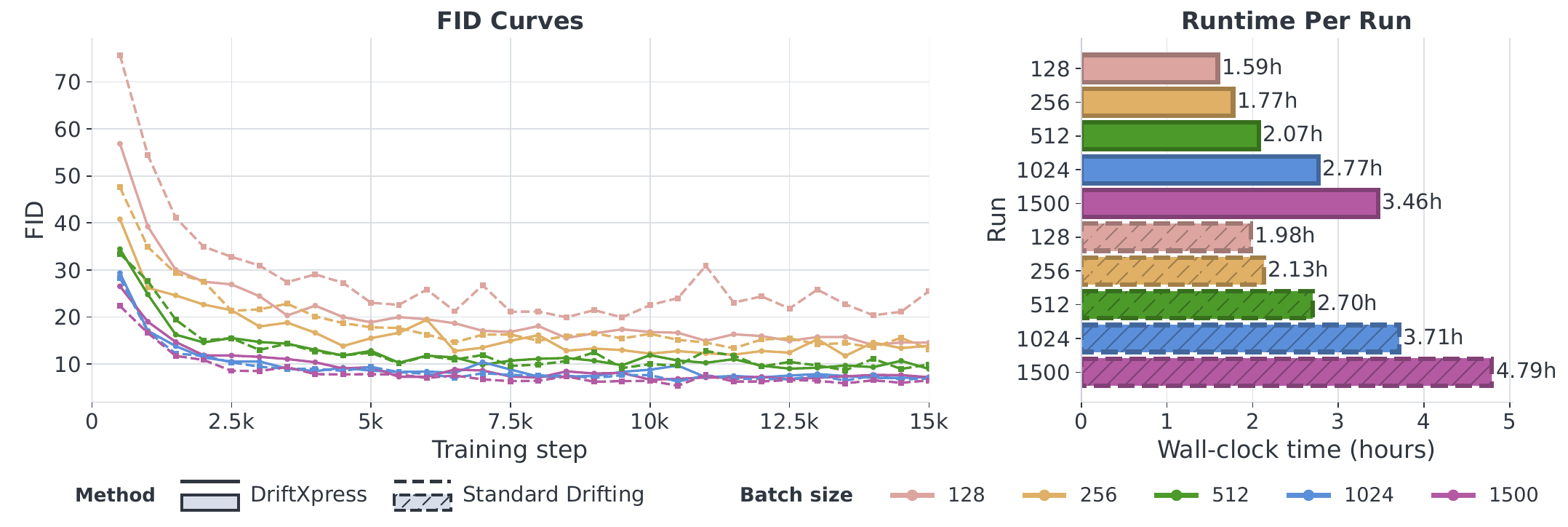}
\caption{\textbf{CIFAR10 batch-size sweep.} (left): FID trajectories for \method\ and standard drifting across batch sizes $128$--$1500$. (right): wall-clock runtime for each run. \method{} outperforms standard drifting across all batch sizes. For larger batch sizes, the primary benefit is runtime, whereas for small batches, the advantage is faster convergence and improved training trajectories.  }
\label{fig:cifar10-batch-size-sweep}
    \end{figure}

\begin{figure}[t]
    \vspace{-0.75em}
    \centering
    \includegraphics[width=\linewidth]{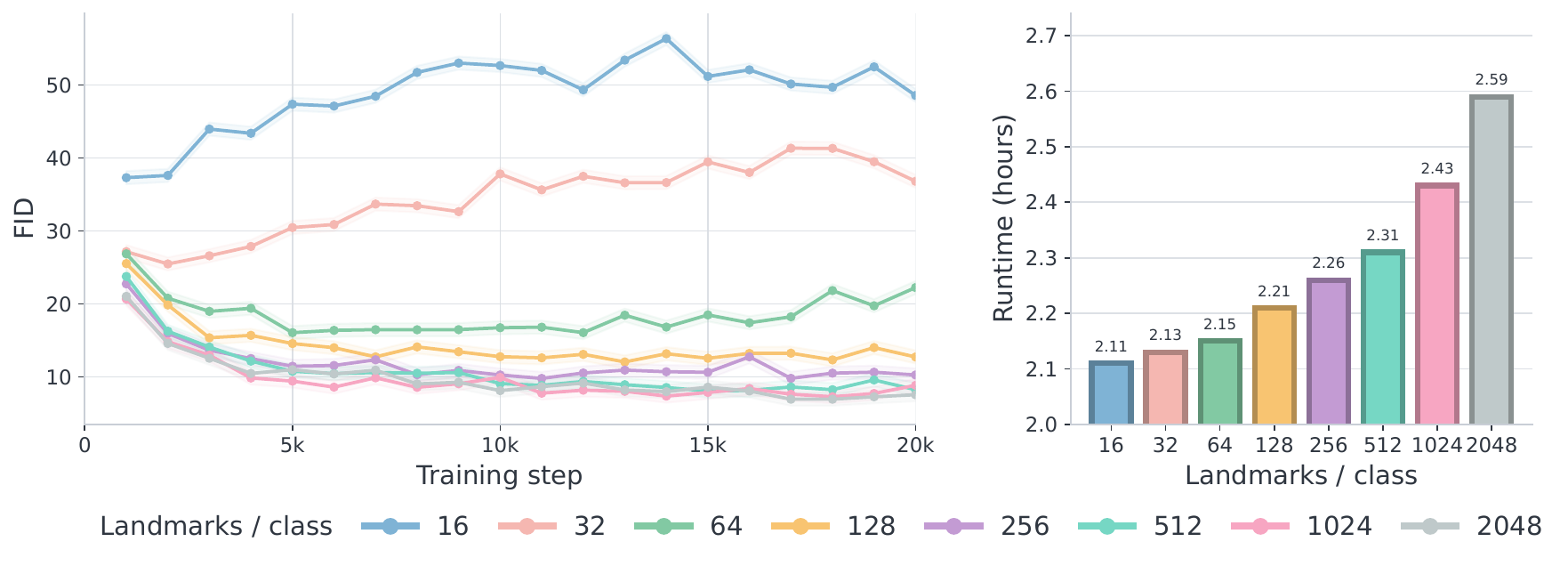}
    \caption{\textbf{Landmark ratio ablation.} (left) Training FID for different numbers of landmarks per class. (right) total training time. While higher landmark ratios improve FID, diminishing returns in sample quality against increasing runtimes suggest an optimal balance between $0.02$ and $0.1$.}
    \label{fig:landmark_quantity_ablation}
\vspace{-1.25em}
\end{figure}

\subsection{Wall-Clock Convergence}
\label{sec:wall_clock_convergence}
Figure~\ref{fig:wallclock} plots FID against elapsed training time for SVHN, CIFAR10 (both unsharded), and CIFAR100 (sharded). Across all three datasets, \method{} (green) improves the early- and mid-training FID trajectories relative to standard drifting (orange). This trend is also visible qualitatively in Figure~\ref{fig:speed}: at the same training snapshots, \method{} produces recognizable image structure much earlier than standard drifting. This is because the projected kernel summary uses a larger effective support set than the mini-batches used by standard drifting, yielding a smoother and more stable field estimate. Consequently, \method{} reaches low-FID regimes substantially faster in wall-clock time. The effect is most pronounced on SVHN and CIFAR10, while CIFAR100 has a more gradual trajectory due to its higher dataset complexity and use of the sharded variant. Overall, these trajectories support the main efficiency claim: the projected field not only reduces per-step cost, but also improves the rate at which drifting reaches competitive sample quality under a fixed compute budget.

\subsection{Batch-Size Sweep} Figure~\ref{fig:cifar10-batch-size-sweep} compares \method{} (solid lines) to standard drifting (dashed lines) across batch sizes from $128$ to $1500$. The left panel tracks CIFAR10 FID trajectories over $15{,}000$ steps, while the right panel reports end-to-end wall-clock runtimes. 


\method{} consistently outperforms standard drifting, with the specific advantage varying by regime. In large-batch regimes (typical of GPUs with larger memory), the primary benefit is computational efficiency. Because large batches naturally stabilize field estimation, both methods yield comparable trajectories, but \method{} is significantly faster: runtime drops by $25.3\%$ ($0.94$ hours) at batch size $1024$ (in blue), and by $27.7\%$ ($1.33$ hours) at $1500$ (in purple). This widening efficiency gap accounts for the major speedups observed in Table~\ref{tab:main_results_big}. Since the runtime delta grows from $0.39$ hours at batch size $128$ to $1.33$ hours at $1500$, extrapolating to our main benchmark's global batch size of $12{,}000$ is consistent with the observed $6.68\times$ maximum speedup. Conversely, in memory-constrained small-batch regimes, the primary advantage is faster convergence. For batch sizes of $128$ (in red) and $256$ (in orange), \method{} yields strictly better FID trajectories alongside moderate runtime reductions of $19.7\%$ and $17.0\%$. This superior convergence occurs because \method{} evaluates its attractive field against precomputed global summaries of the entire training set, maintaining a stable view of the landscape rather than relying on the limited, noisy positives of a small mini-batch.


\subsection{Landmark Selection Strategies} Table~\ref{tab:nystrom_landmark_selection_time} compares landmark selection strategies on CIFAR10, evaluated over $10$k training steps with a batch size of $1024$. We report the best FID scores for four primary methods: random sampling, $k$-means, $k$-center, and facility location (see App.~\ref{app:landmark-quality} for details). We evaluate these strategies across two variants: global (selected from the full training set) and per-class (selected equally within each class). These are tested at matched budgets of $1280$ total landmarks ($128$ per class) and $5120$ total landmarks ($512$ per class). Weighted $k$-center is evaluated per-class exclusively, as a global version would confound density weighting with class imbalance. 

\label{sec:landmark_selection}


We observe that most methods perform similarly (FID between $9 - 12$) with the notable exception of $k$-center, which significantly underperforms (FID between $14 - 29$). The overall best FID is achieved by global $k$-means, scoring $11.26 \pm 0.89$ and $9.37 \pm 0.37$ at the low and high budgets, respectively, while maintaining moderate computational cost. Facility location is competitive under the higher $ 5120$-budget ($9.49 \pm 0.84$ globally) but is computationally expensive. Meanwhile, random selection is substantially cheaper yet remains highly competitive, particularly in the per-class setting, where it achieves an FID of $9.60 \pm 0.95$ with $512$ landmarks per class. This aligns with prior Nystr\"om approximation findings by \citep{kumar2009sampling}, who demonstrated that uniform sampling without replacement provides effective approximations with minimal time and memory overhead. Because random selection offers the best balance of computational speed and competitive sample quality, we adopt random per-class landmarks as the default for all main experiments.

\subsection{Landmark Ratio}
\label{sec:landmark_quantity}
Figure~\ref{fig:landmark_quantity_ablation} evaluates the effect of landmark ratio in \method{} on CIFAR10 using random per-class landmark selection. Since CIFAR10 has $5000$ training samples per class, varying the number of landmarks per class from $16$ to $2048$ corresponds to landmark ratios from $0.0032$ to $0.4096$. The left panel reports FID trajectories over $20$k training steps and the right panel reports the corresponding end-to-end wall-clock runtime. We observe a clear quality-runtime trade-off. At the lowest ratio ($0.0032$), training is fast ($2.11$ hours) but unstable, ending with a poor FID of $48.57$. Conversely, the highest ratio ($0.4096$) yields a strong $7.57$ FID but increases the runtime to $2.59$ hours. This reflects the increased computational cost of query-landmark features and projected barycenters required for a larger, more accurate Nystr\"om approximation. 
However, quality gains show diminishing returns. Increasing the ratio to the $0.0256$--$0.1024$ range significantly improves FID, but further increases yield only marginal benefits given the computational overhead. Once the basis captures the distribution's dominant geometry, extra landmarks inflate the runtime. Consequently, the optimal trade-off lies between ratios of $0.0256$ and $0.1024$ ($128$--$512$ landmarks per class).

    \vspace{-0.35em}

\section{Conclusion}

    \vspace{-0.35em}


Drifting models replace iterative sampling with one-step generation, shifting computational burden to training. The main cost is field estimation: standard drifting repeatedly computes kernel interactions between generated samples and the support defining the attraction--repulsion field. \method{} addresses this by projecting attraction onto an RKHS subspace, thereby replacing repeated interactions with cached summaries. This provides a twofold speedup. First, wall-clock convergence improves because each projected step uses a full-dataset cached summary rather than mini-batch positive samples, yielding a smoother field that reaches low-FID regions faster. Second, it reduces the per-step field-estimation cost by replacing exact dataset attraction with query-landmark kernels and low-rank summary-multiplication. Across SVHN, CIFAR10, CIFAR100, and ImageNet, \method{} achieves comparable sample quality to standard drifting while substantially reducing training time.


The method has several limitations. Projecting attraction while keeping repulsion exact improves stability, but breaks exact anti-symmetry and leaves a quadratic batch-size dependence in the repulsive term. Our theory bounds the resulting projected-attraction error through the landmark kernel residual and local kernel mass, but safely approximating repulsion remains open. \method{} also assumes a fixed feature encoder, since changing the encoder would invalidate the cached summaries. These limitations point to generated-sample-aware repulsion approximations and more memory-efficient projected summaries as next steps.

\section*{Acknowledgments}

The authors thank Mingyang Deng for helpful discussions and clarifications regarding the original Drifting Models paper. Resources were provided, in part, by the Province of Ontario, the Government of Canada through CIFAR, and companies sponsoring the Vector Institute.

\bibliographystyle{alpha}
\bibliography{template/biblio}

\clearpage

\appendix
\section*{Appendix Contents}
\startcontents[appendix]
\printcontents[appendix]{}{1}{\setcounter{tocdepth}{2}}

\setcounter{theorem}{0}

\section{Related Work}

\paragraph{Diffusion and flow-based generative models.}
Diffusion models and score-based generative models learn to transform noise into data by reversing a gradual noising process or by solving a reverse-time stochastic differential equation~\citep{sohl2015deep,ho2020denoising,song2020score}. Flow-based generative models and stochastic interpolants instead learn deterministic or stochastic transport dynamics between simple and data distributions, often through ordinary differential equations or velocity fields~\citep{lipman2022flow,liu2022flow,albergo2023stochastic}. These methods have achieved strong sample quality, but their generation procedure requires many function evaluations. At inference time, generation is therefore tied to an iterative solver that repeatedly updates a sample along a learned trajectory. 

\paragraph{One-step and few-step generation.}
A large body of work has aimed to reduce the sampling cost of diffusion and flow-based models. Distillation-based methods compress a pretrained multi-step teacher into a generator that uses fewer sampling steps~\citep{salimans2022progressive,yin2024one,zhou2024score}. Consistency models train mappings that are consistent along the probability-flow trajectory and can generate in one or a few steps, either by distillation or standalone training~\citep{song2023consistency}. More recent approaches train one-step or few-step models from scratch by modifying the underlying flow or diffusion objective, for example, through shortcut models, mean flows, flow map matching, or moment-matching self-distillation~\citep{frans2024one,geng2025mean,boffi2024flow,zhou2025inductive}. These methods remain closely tied to diffusion or flow trajectories: they either distill a pretrained trajectory, enforce consistency along one, or learn a modified velocity/flow representation. 

\paragraph{Generative adversarial networks.}
Generative adversarial networks (GANs) train a generator and discriminator through a minimax objective~\citep{goodfellow2014generative}. Like GANs, Drifting Models use a generator that maps noise to samples in a single forward pass. However, the training signal is fundamentally different. GANs learn through adversarial discrimination, which can introduce optimization instability, mode collapse, and sensitivity to training dynamics~\citep{salimans2016improved,arjovsky2017towards,mescheder2018which}. Drifting Models avoid an adversarial discriminator and instead define an attraction--repulsion field whose equilibrium is reached when the generated and data distributions match~\citep{deng2026drifting}. 

\paragraph{Variational autoencoders and normalizing flows.}
Variational autoencoders (VAEs) optimize an evidence lower bound with an encoder-decoder architecture and a latent prior~\citep{kingma2013auto}. They provide efficient ancestral sampling from the prior, but their objective and inductive biases differ from drift-based training. Modern generative systems also often use autoencoding models as latent tokenizers or compression modules for diffusion or autoregressive models rather than as standalone high-fidelity generators~\citep{esser2021taming,rombach2022high}. Normalizing flows learn invertible transformations and optimize exact likelihood through change-of-variables formulas~\citep{rezende2015variational,dinh2016density}. They admit exact sampling and likelihood evaluation, but require architectural invertibility and tractable Jacobian determinants. 

\paragraph{Moment matching and kernel generative models.}
Kernel moment-matching methods compare distributions through discrepancies between kernel mean embeddings. Maximum Mean Discrepancy (MMD) measures the largest difference in expectations over the unit ball of an RKHS~\citep{gretton2012kernel}, and generative moment matching networks use MMD as a training objective for feedforward generators~\citep{dziugaite2015training,li2015generative}. More recent work has revisited moment matching for fast generative modeling, including one-step and few-step models~\citep{zhou2025inductive}. DriftXpress is related in its use of kernels and positive/negative samples, but the role of the kernel is different. Moment-matching methods usually define a scalar discrepancy between distributions. Drifting Models instead define a local vector field that moves each generated sample through attraction toward positive samples and repulsion from negative samples.

\paragraph{Kernel smoothing and mean-shift structure.}
The normalized attraction term in Drifting Models is closely related to classical kernel smoothing. In particular, the map
\[
  \boldsymbol{\mu}_p(\mathbf x)
  =
  \frac{\sum_j k(\mathbf x,\mathbf y_j^+)\mathbf y_j^+}
       {\sum_j k(\mathbf x,\mathbf y_j^+)}
\]
has the form of a Nadaraya--Watson kernel regression estimator, which predicts a locally weighted average using normalized kernel weights~\citep{nadaraya1964estimating,watson1964smooth}. In the drifting setting, the response values are the data features themselves, so this estimator defines a local barycenter that pulls a generated sample toward nearby data. This also connects drifting to mean shift, where kernel-weighted local means induce updates that move points toward modes of a kernel density estimate~\citep{cheng1995mean,comaniciu2002mean}. The distinction is that Drifting Models do not perform pure density-mode seeking: they combine attraction toward the data distribution with repulsion from the generated distribution, yielding a signed attraction--repulsion field.

\paragraph{RKHS embeddings, Nystr\"om approximation, and scalable kernel discrepancies.}
RKHS methods provide a classical way to study nonlinear learning problems through linear geometry in function space. This perspective underlies support vector machines, kernel PCA, kernel mean embeddings, MMD, and related nonparametric estimators~\citep{cortes1995support,scholkopf1998nonlinear,smola2007hilbert,gretton2012kernel,muandet2017kernel}. Kernel mean embeddings represent probability distributions as RKHS elements and support two-sample testing, independence testing, conditional mean embeddings, and distributional learning~\citep{smola2007hilbert,sriperumbudur2010hilbert,muandet2017kernel}. Since exact kernel methods often require large pairwise kernel matrices, a long line of work studies scalable approximations. The Nystr\"om method approximates a kernel matrix or operator using a subset of landmark columns~\citep{williams2001using,drineas2005nystrom}, and related ideas have been used to accelerate kernel discrepancy methods such as MMD and Kernel Stein Discrepancy~\citep{gretton2012kernel,chatalic2025scalable,liu2016kernelized,chwialkowski2016kernel,kalinke2024nystr}. DriftXpress is closest in spirit to these scalable kernel methods, but differs in its object of approximation: rather than accelerating a kernel classifier, eigendecomposition, or scalar discrepancy, it projects drifting kernel sections onto a landmark-induced RKHS subspace. This yields finite-dimensional features that allow the attractive field to be written using precomputed summaries over the training set, preserving the normalized kernel-field form of Drifting Models while reducing repeated exact interactions with all positive samples.

\paragraph{Recent work on Drifting Models.}
Drifting Models~\citep{deng2026drifting} have already prompted several follow-up studies that analyze, reinterpret, or extend the drifting field. A first line of work studies the mathematical structure of the field itself. \citep{turan2026generative} shows that, under a Gaussian kernel, the drift operator can be written as a score difference between smoothed distributions, connecting drifting to score matching, spectral convergence, and Wasserstein gradient-flow interpretations. Similarly, \citep{cao2026gradient} proves an equivalence between drifting and the Wasserstein gradient flow of a KDE-approximated forward KL divergence, and further relates MMD-based generators to the same broader family of KDE-approximated divergence flows. \citep{franz2026drifting} study whether drifting corresponds to optimizing a scalar potential, showing that drift fields are generally non-conservative due to the position-dependent normalization, while Gaussian kernels form a special conservative case. Complementary theoretical work studies identifiability and stability: \citep{lee2026identifiability} analyzes companion-elliptic kernel families, including the Laplace kernel, and \citep{kazanskii2026attraction} introduces a friction-augmented variant that addresses local instability in simplified drifting dynamics.

\section{Algorithm}

We provide the algorithms for \method{} preprocessing and training in Algorithms~\ref{alg:driftxpress-preprocess} and ~\ref{alg:driftxpress-train} respectively. 

\begin{algorithm}[htb]
\caption{DriftXpress: Preprocessing}
\label{alg:driftxpress-preprocess}
\begin{algorithmic}[1]
\REQUIRE Training data $\{\mathbf y^+_j\}_{j=1}^{N^+}$, landmark budget $r$, kernel $k$, regularizer $\lambda$
\ENSURE Landmarks $\mathbf U$, Nystr\"om transform $\mathbf W \in \mathbb{R}^{r \times r}$, attractive summaries $\mathbf A_p \in \mathbb{R}^{r \times D}$, $\mathbf b_p \in \mathbb{R}^{r}$

\STATE \textbf{// Landmark selection}
\STATE Sample landmarks $\mathbf U = \{\mathbf u_1, \ldots, \mathbf u_r\}$ from $\{\mathbf y^+_j\}_{j=1}^{N^+}$ \COMMENT{e.g., random per-class sampling}

\STATE \textbf{// Nystr\"om transform}
\STATE Compute landmark Gram matrix $[\mathbf K_{UU}]_{ab} \leftarrow k(\mathbf u_a, \mathbf u_b)$ for $a,b \in [r]$
\STATE Compute $\mathbf W \leftarrow (\mathbf K_{UU} + \lambda \mathbf I)^{-1/2}$ \COMMENT{Stored once; used throughout training}

\STATE \textbf{// Precompute attractive summaries over the full training set}
\STATE Initialize $\mathbf A_p \leftarrow \mathbf{0}^{r \times D}$, $\mathbf b_p \leftarrow \mathbf{0}^{r}$
\FOR{$j = 1, \ldots, N^+$}
    \STATE $\mathbf K_{\mathbf y^+_j U} \leftarrow [k(\mathbf y^+_j, \mathbf u_1), \ldots, k(\mathbf y^+_j, \mathbf u_r)] \in \mathbb R^{1\times r}$
    \STATE $\boldsymbol{\varphi}(\mathbf y^+_j) \leftarrow \mathbf W^\top \mathbf K_{\mathbf y^+_j U}^\top \in \mathbb{R}^{r}$
    \STATE $\mathbf A_p \leftarrow \mathbf A_p + \boldsymbol{\varphi}(\mathbf y^+_j)\, \mathbf y^{+\top}_j$
    \STATE $\mathbf b_p \leftarrow \mathbf b_p + \boldsymbol{\varphi}(\mathbf y^+_j)$
\ENDFOR

\STATE \RETURN $\mathbf U,\ \mathbf W,\ \mathbf A_p,\ \mathbf b_p$
\end{algorithmic}
\end{algorithm}

\begin{algorithm}[htb]
\caption{DriftXpress: Training}
\label{alg:driftxpress-train}
\begin{algorithmic}[1]
\REQUIRE Generator $f_\theta$, noise distribution $p_\epsilon$, precomputed $\mathbf U,\mathbf W,\mathbf A_p,\mathbf b_p$, kernel $k$, positive normalization offset $\varepsilon>0$, learning rate $\eta$, training steps $T$
\ENSURE Trained generator $f_\theta$

\FOR{$i = 1, \ldots, T$}

    \STATE \textbf{// Sample generated batch}
    \STATE Sample $\{\epsilon_b\}_{b=1}^B \sim p_\epsilon$
    \STATE $\mathbf x_b \leftarrow f_\theta(\epsilon_b)$ for $b \in [B]$ \COMMENT{Generated samples from the current model distribution}

    \STATE \textbf{// Compute projected attraction from cached training-data summaries}
    \FOR{$b = 1, \ldots, B$}
        \STATE $\mathbf K_{\mathbf x_b U} \leftarrow [k(\mathbf x_b, \mathbf u_1), \ldots, k(\mathbf x_b, \mathbf u_r)] \in \mathbb R^{1\times r}$
        \STATE $\boldsymbol{\varphi}_b \leftarrow \mathbf W^\top \mathbf K_{\mathbf x_b U}^\top \in \mathbb{R}^{r}$
        \STATE $\mu^U_p(\mathbf x_b) \leftarrow \dfrac{\mathbf A_p^\top \boldsymbol{\varphi}_b}{\boldsymbol{\varphi}_b^\top \mathbf b_p+\varepsilon}$ \COMMENT{Projected attractive mean in $\mathbb R^D$}
    \ENDFOR

    \STATE \textbf{// Compute exact repulsion over the current generated batch}
    \FOR{$b = 1, \ldots, B$}
        \STATE $K^-_{b\ell} \leftarrow k(\mathbf x_b, \mathbf x_\ell)$ for $\ell \in [B]$
        \STATE $K^-_{bb} \leftarrow 0$ \COMMENT{Remove self-interaction}
        \STATE $\mu_q(\mathbf x_b) \leftarrow \dfrac{\sum_{\ell=1}^{B} K^-_{b\ell}\mathbf x_\ell}{\sum_{\ell=1}^{B} K^-_{b\ell}+\varepsilon}$ \COMMENT{Exact repulsive mean}
    \ENDFOR

    \STATE \textbf{// Assemble the DriftXpress field}
    \FOR{$b = 1, \ldots, B$}
        \STATE $\mathbf V^{\mathbf U}_{p,q}(\mathbf x_b) \leftarrow \mu^U_p(\mathbf x_b) - \mu_q(\mathbf x_b)$
    \ENDFOR

    \STATE \textbf{// Compute loss and update}
    \STATE $\mathcal{L} \leftarrow \dfrac{1}{B}\sum_{b=1}^B \left\| f_\theta(\epsilon_b) - \mathrm{stopgrad}\!\left(\mathbf x_b + \mathbf V^{\mathbf U}_{p,q}(\mathbf x_b)\right) \right\|_2^2$
    \STATE $\theta \leftarrow \theta - \eta \nabla_\theta \mathcal{L}$

\ENDFOR

\STATE \RETURN $\theta$
\end{algorithmic}
\end{algorithm}

\subsection{Complexity Analysis}
\label{sec:complexity}

We analyze the cost of estimating the drifting field. Let $B$ be the number of generated query samples per step, $N^+$ the number of training samples used for attraction, $N^-$ the number of generated samples used for repulsion, $D$ the feature dimension, and $r$ the number of landmarks.

\paragraph{Standard Drift.} Standard drifting computes both attraction and repulsion through exact kernel evaluations. For each query $\mathbf x_b$, the attractive component requires the weights
\[
  k(\mathbf x_b,\mathbf y_j^+),
  \qquad j=1,\dots,N^+,
\]
and the repulsive component requires the weights
\[
  k(\mathbf x_b,\mathbf y_\ell^-),
  \qquad \ell=1,\dots,N^-.
\]
Thus, for $B$ query samples, standard drifting must form $BN^+$ attraction kernels and $BN^-$ repulsion kernels at every training step. Since each kernel evaluation requires computing a distance in $D$-dimensional feature space, the per-step field-estimation cost is
\[
  \mathcal O\!\left(BN^+D + BN^-D\right),
\]
with memory
\[
  \mathcal O\!\left(BN^+ + BN^-\right)
\]

\paragraph{DriftXpress.}
DriftXpress replaces the exact attractive-side kernel evaluations with a Nystr\"om approximation. Given landmarks
\[
  \mathbf U=\{\mathbf u_1,\dots,\mathbf u_r\}\subset\R^D,
\]
we form
\[
  [\mathbf K_{UU}]_{ab}=k(\mathbf u_a,\mathbf u_b),
  \qquad
  \mathbf W=(\mathbf K_{UU}+\lambda\mathbf I)^{-1/2}.
\]
This preprocessing costs
\[
  \mathcal O(r^2D+r^3),
\]
where $r^2D$ forms the landmark Gram matrix and $r^3$ computes the inverse square root. For any point $\mathbf z$, the Nystr\"om feature is
\[
  \boldsymbol{\varphi}(\mathbf z)
  =
  \mathbf K_{\mathbf zU}\mathbf W,
  \qquad
  \mathbf K_{\mathbf zU}
  =
  [k(\mathbf z,\mathbf u_1),\dots,k(\mathbf z,\mathbf u_r)].
\]
Computing one feature costs $\mathcal O(rD+r^2)$.

The attractive summaries are precomputed once:
\[
  \mathbf A_p
  =
  \sum_{j=1}^{N^+}
  \boldsymbol{\varphi}(\mathbf y_j^+)\mathbf y_j^{+\top}
  \in\R^{r\times D},
  \qquad
  \mathbf b_p
  =
  \sum_{j=1}^{N^+}
  \boldsymbol{\varphi}(\mathbf y_j^+)
  \in\R^r.
\]
The one-time attractive-side preprocessing cost is
\[
  \mathcal O\!\left(r^2D+r^3+N^+(rD+r^2)\right).
\]
At training time, for $B$ generated queries, DriftXpress computes their projected features and evaluates
\[
  \boldsymbol{\mu}^{\mathbf U}_p(\mathbf x)
  =
  \frac{\boldsymbol{\varphi}(\mathbf x)^\top\mathbf A_p}
       {\boldsymbol{\varphi}(\mathbf x)^\top\mathbf b_p}.
\]
This costs
\[
  \mathcal O\!\left(B(rD+r^2)\right)
\]
for the projected attractive component.

The repulsive component is unchanged and is computed exactly:
\[
  \boldsymbol{\mu}_q(\mathbf x)
  =
  \frac{\sum_{\ell=1}^{N^-}k(\mathbf x,\mathbf y^-_\ell)\mathbf y^-_\ell}
       {\sum_{\ell=1}^{N^-}k(\mathbf x,\mathbf y^-_\ell)}.
\]
Thus exact repulsion costs
\[
  \mathcal O(BN^-D)
\]
per step. Overall, DriftXpress has per-step field-estimation cost
\[
  \mathcal O\!\left(B(rD+r^2)+BN^-D\right),
\]
compared with
\[
  \mathcal O\!\left(BN^+D+BN^-D\right)
\]
for standard drifting. Therefore, DriftXpress removes the repeated dependence on the training support size $N^+$ and replaces it with dependence on the landmark dimension $r$, while preserving exact repulsion. The gain is largest when $r\ll N^+$ and attraction is a substantial part of the field-estimation cost.

\subsection{Memory and Sharding}
\label{app:shard}
Treating the number of feature groups and temperatures as constants, standard drifting
materializes exact attraction and repulsion weights of sizes $B\times N^+$ and
$B\times N^-$. Its dominant batchwise memory is therefore
\[
  \mathcal O(BN^+ + BN^-).
\]

\method{} replaces the exact attractive kernel with a Nystr\"om cache. With $r$
landmarks, the attractive summaries store
\[
  \mathbf A_p\in\mathbb R^{r\times D},
  \qquad
  \mathbf W\in\mathbb R^{r\times r},
\]
and each step materializes query features of size $B\times r$. Thus the dominant
memory is
\[
  \mathcal O(rD+r^2+Br+BN^-),
\]
or more precisely $\mathcal O(T(rD+r^2)+Br+BN^-)$ when the number of temperatures
$T$ is not treated as constant.

For large landmark banks, we split the attractive cache into $S$ shards with
$r_s$ landmarks per shard, where $r=\sum_{s=1}^S r_s$. Sharding is possible
because the attractive numerator and denominator decompose additively across
shards. The total resident attractive cache becomes
\[
  \mathcal O\!\left(\sum_{s=1}^S (r_sD+r_s^2)\right)
  =
  \mathcal O\!\left(rD+\sum_{s=1}^S r_s^2\right).
\]
Under sequential shard evaluation, the active query-feature workspace is reduced
from $\mathcal O(Br)$ to
\[
  \mathcal O\!\left(\max_{s\in[S]} Br_s\right).
\]
Therefore, when all shards remain resident on GPU, the peak memory is
\[
  \mathcal O\!\left(
    rD+\sum_{s=1}^S r_s^2
    +\max_{s\in[S]}Br_s
    +BN^-
  \right),
\]
up to common output and accumulator buffers. If shards are streamed one at a time
from CPU or off-rank storage, the peak memory can instead be reduced to
\[
  \mathcal O\!\left(
    \max_{s\in[S]}(r_sD+r_s^2+Br_s)
    +BN^-
  \right).
\]

Sharding therefore trades throughput for memory. It does not reduce the total
linear summary storage $rD$ when all shards remain resident, but it reduces the
quadratic Nystr\"om term from $r^2$ to $\sum_s r_s^2$ and avoids materializing the
full $B\times r$ query-feature workspace at once. For balanced shards, this
changes the quadratic term from $\mathcal O(r^2)$ to $\mathcal O(r^2/S)$. In the
class-wise setting, with $C$ classes and $m$ landmarks per class, the unsharded
quadratic term scales as $\mathcal O(C^2m^2)$, while class-wise sharding reduces
it to $\mathcal O(Cm^2)$. The cost is sequential accumulation over shards, which
lowers throughput relative to the unsharded variant whenever both fit in memory.

\section{Additional Experiments}
\label{app:additional_exp}

\subsection{Experimental Setup}
\label{app:experimental-setup}

We use the same generator, feature encoder, optimizer, and evaluation protocol for all reported runs unless otherwise stated. The goal of the experimental protocol is to isolate the effect of the field estimator: standard drifting and \method{} use the same generator architecture, the same drifting objective, the same frozen feature encoder, and the same optimization hyperparameters. The only difference is that standard drifting estimates attraction with exact kernel interactions, whereas \method{} replaces the attractive component with a Nystr\"om-projected RKHS field and cached positive summaries.

\paragraph{Datasets.}
We evaluate on SVHN, CIFAR10, CIFAR100, and ImageNet. CIFAR10 and CIFAR100 contain $50{,}000$ training images, with $10$ and $100$ classes, respectively. SVHN contains $73{,}257$ training images and $10$ imbalanced classes. ImageNet contains $1{,}281{,}167$ training images and $1{,}000$ imbalanced classes. All datasets are used with their standard training splits for training and reference statistics. 

\paragraph{Generator architecture.}
Every reported run uses the same U-Net generator. The generator has base channel width $128$, channel multipliers $(1,2,2,2)$, two residual blocks per resolution, self-attention at spatial resolution $16 \times 16$, dropout $0.1$, and four attention heads. For RGB datasets, this architecture has $38.3$M trainable parameters. Since \method{} changes only the training-time field estimator, the generator architecture and one-step inference procedure are identical to standard drifting.

\paragraph{Feature encoder.}
The drifting field is computed in the representation space of a frozen DINOv3 feature encoder. We use input resolution $112 \times 112$ and pooled spatial size $4$. The encoder has $85.6$M frozen parameters. The encoder is used only during training to define the feature-space drifting objective. It is not used at inference time, where generation requires only a single forward pass through the trained U-Net generator.

\paragraph{Optimization.}
All methods are trained with AdamW using learning rate $2 \times 10^{-4}$, $\beta_1 = 0.9$, $\beta_2 = 0.999$, and zero weight decay. We use EMA with decay $0.9999$, gradient clipping at norm $2.0$, fused optimizer updates, and bfloat16 autocast. Both the generator and the frozen feature encoder are compiled in the reported runs.

\paragraph{Drifting field configuration.}
Standard drifting computes both attraction and repulsion using exact kernel-weighted barycenters. In contrast, \method{} uses a Nystr\"om approximation only for the attractive component and keeps the repulsive component exact. Landmarks are selected from the training data before training begins. Unless otherwise stated, we use random per-class landmark selection. 

\paragraph{Sharding.}
For datasets with many classes or large landmark banks, we use the sharded variant of \method{}. Sharding partitions the positive summaries by class and accumulates the attractive numerator and denominator sequentially across shards. In the main benchmarks, we use one shard per class: $10$ shards for SVHN and CIFAR10, $100$ shards for CIFAR100, and $1000$ shards for ImageNet. This reduces the active memory required for projected attraction while preserving the same additive attractive field.

\paragraph{Compute.}
Main benchmark experiments are run on $8$ NVIDIA H100 GPUs. Ablations and profiling experiments are run on a single NVIDIA H100 GPU unless otherwise specified. The main benchmark uses a fixed wall-clock budget of $8$ hours and a global batch size of $12{,}000$, corresponding to $1500$ samples per GPU. We report iteration speed, image throughput, wall-clock runtime, peak GPU memory, and relative speedup against standard drifting under the same dataset and training protocol.

\paragraph{Random seeds and error bars.}
All main quantitative results are reported as mean $\pm$ standard deviation across three independent seeds. The seeds affect model initialization, data ordering, noise sampling, and landmark selection when random landmarks are used. For each run, we track FID throughout training and report \emph{Best FID}, defined as the lowest FID achieved during the run.

\paragraph{Kernel and temperature.}
All drifting fields use the Laplace kernel in the frozen feature space. For feature vectors
$x,y \in \mathbb{R}^D$, we define
\[
    k_\tau(x,y)
    =
    \exp\left(-\frac{\|x-y\|_2}{\tau}\right),
\]
where $\tau>0$ is the kernel temperature. We use $\tau=0.05$ for all reported experiments. The same kernel and temperature are used for standard drifting and \method{}. 

\paragraph{FID evaluation.}
We evaluate sample quality using Fr\'echet Inception Distance (FID)~\citep{heusel2017gans}. FID is computed with the standard Inception-v3 feature extractor. For each evaluation checkpoint, we generate samples with the one-step generator, extract Inception features for generated and real images, fit Gaussian statistics to both feature sets, and report the resulting FID. We use the same FID code, preprocessing, reference set, number of generated samples, and evaluation schedule for standard drifting and \method{}.

For CIFAR10, CIFAR100, and SVHN, FID is computed using a fixed $50{,}000$-image real reference set from the corresponding dataset and $50{,}000$ generated samples. The reported \emph{Best FID} is the lowest FID observed over the evaluation checkpoints of a training run. For ImageNet, FID is computed against the $50{,}000$ images from the ILSVRC2012 validation split as the real reference distribution. We generate $50{,}000$ samples from the trained generator. For class-conditional ImageNet generation, labels are sampled uniformly over the $1000$ classes, corresponding to $50$ generated images per class.

\begin{table}[t]
\centering
\small
\setlength{\tabcolsep}{6pt}
\caption{Dataset statistics.}
\label{tab:dataset-details}
\begin{tabular}{lcccc}
\toprule
Dataset & Training samples & Classes & Class Balance & Samples per class \\
\midrule
CIFAR10 & $50{,}000$ & $10$ & Balanced & $5{,}000$ \\
CIFAR100 & $50{,}000$ & $100$ & Balanced & $500$ \\
SVHN & $73{,}257$ & $10$ & Imbalanced & $4{,}659$--$13{,}861$ \\
ImageNet & $1{,}281{,}167$ & $1{,}000$ & Imbalanced & $732$--$1{,}300$ \\
\bottomrule
\end{tabular}
\end{table}

\subsection{Landmark Selection Algorithms}
\label{app:landmark-quality}

We compare four landmark-selection strategies: random sampling, $k$-means, $k$-center, and facility location. These methods differ in whether they prioritize simplicity, density-aware representativeness, worst-case geometric coverage, or average-case similarity coverage. The main paper reports the resulting FID and landmark-selection time for CIFAR10, where global $k$-means gives the strongest performance among the tested methods. 

\paragraph{Random sampling.}
Random sampling selects $M$ landmarks uniformly without replacement:
\[
  \mathbf U \sim \operatorname{Unif}\{\mathbf U\subseteq\mathcal Y:|\mathbf U|=M\}.
\]
It is computationally cheap and unbiased with respect to the empirical distribution, but it does not optimize coverage or reduce redundancy.

\paragraph{$k$-means.}
$k$-means first solves
\[
  \min_{\mathbf c_1,\dots,\mathbf c_M}
  \sum_{i=1}^N
  \min_{m\in[M]}
  \|\mathbf y_i-\mathbf c_m\|_2^2,
\]
and then selects the nearest real feature to each centroid:
\[
  \mathbf u_m
  =
  \arg\min_{\mathbf y_i\in\mathcal Y}
  \|\mathbf y_i-\mathbf c_m\|_2.
\]
This gives a medoid-like landmark set with $\mathbf U\subseteq\mathcal Y$. Because the objective is density-weighted, $k$-means places more landmarks in high-mass regions of the feature distribution. This is well aligned with \method{}, where the field is determined by normalized kernel averages rather than worst-case coverage.

\paragraph{$k$-center.}
$k$-center approximates the minimax covering objective
\[
  \min_{\mathbf U\subseteq\mathcal Y,\ |\mathbf U|=M}
  \max_{\mathbf y_i\in\mathcal Y}
  \min_{\mathbf u\in\mathbf U}
  \|\mathbf y_i-\mathbf u\|_2.
\]
The greedy farthest-point rule adds
\[
  \mathbf u_{t+1}
  =
  \arg\max_{\mathbf y_i\in\mathcal Y}
  \min_{\mathbf u\in\mathbf U_t}
  \|\mathbf y_i-\mathbf u\|_2.
\]
This prioritizes worst-case geometric spread. While useful for covering radius, it can overallocate landmarks to isolated or low-density points. For \method{}, this is not necessarily desirable because the projected field depends on typical kernel-weighted neighborhoods, not only on the worst-covered feature.

\paragraph{Facility location.}
Facility location maximizes a similarity-based coverage objective:
\[
  F(\mathbf U)
  =
  \sum_{i=1}^N
  \max_{\mathbf u\in\mathbf U}
  s(\mathbf y_i,\mathbf u),
  \qquad
  |\mathbf U|=M,
\]
where $s(\cdot,\cdot)$ is a nonnegative similarity. In our implementation, we use an exponential Euclidean-distance similarity
\[
s(x,u)=\exp\!\left(-\frac{\|x-u\|_2}{\tau}\right),
\]
where $\tau$ is the temperature-scaled bandwidth. Given a candidate set $C$ and an evaluation set $E$, the facility-location objective is
\[
F(S)=\sum_{x\in E}\max_{u\in S}s(x,u).
\]
The greedy update is
\[
u_t
=
\arg\max_{u\in C\setminus S_{t-1}}
\sum_{x\in E}
\left[
s(x,u)-\max_{v\in S_{t-1}}s(x,v)
\right]_+,
\]
and we set $S_t=S_{t-1}\cup\{u_t\}$. Here, the gain measures the additional coverage obtained by adding $u$ beyond the best landmark already selected.

Unlike $k$-center, facility location optimizes average-case similarity coverage. It is therefore more density-aware, but usually more expensive because marginal gains require similarities to many training points.

\subsection{Exact and Projected Repulsion}
\label{app:exact-projected-repulsion}

\method{} approximates the attractive component with a projected RKHS field, but keeps the repulsive component exact. We ablate this design choice by also replacing the repulsive term with the same Nystr\"om approximation used for attraction. Table~\ref{tab:repulsion-ablation} reports the resulting FID.

\begin{table}[t]
\centering
\small
\setlength{\tabcolsep}{7pt}
\caption{Ablation of exact versus projected repulsion. Exact repulsion is the default \method{} implementation. Projecting the repulsive term causes severe collapse, showing that repulsion is substantially more sensitive to approximation error than attraction.}
\label{tab:repulsion-ablation}
\begin{tabular}{lcc}
\toprule
Dataset & Repulsion estimator & Best FID $\downarrow$ \\
\midrule
SVHN & Exact & $3.11 \pm 0.12$ \\
SVHN & Nystr\"om & $347.6 \pm 21.4$ \\
\midrule
CIFAR10 & Exact & $5.52 \pm 0.06$ \\
CIFAR10 & Nystr\"om & $368.3 \pm 18.7$ \\
\midrule
CIFAR100 & Exact & $6.15 \pm 0.09$ \\
CIFAR100 & Nystr\"om & $382.9 \pm 16.2$ \\
\midrule
ImageNet & Exact & $9.21 \pm 0.16$ \\
ImageNet & Nystr\"om & $394.5 \pm 13.8$ \\
\bottomrule
\end{tabular}
\end{table}

The result is clear: the method is highly sensitive to the repulsive component. While projecting the attractive term preserves quality and improves speed, a direct reuse of training data landmarks for repulsion caused severe collapse in our experiments. This suggests that repulsion is more sensitive to approximation error and that safely approximating it may require landmarks adapted to the generated distribution. The attractive term mainly pulls generated samples toward the data manifold, and its approximation can be stabilized by precomputing a large training-data summary. The repulsive term, however, controls local diversity among generated samples. Small errors in this term can remove or distort the forces that separate nearby generated samples, causing them to concentrate in a small region of feature space.

This ablation motivates our default design: \method{} applies the Nystr\"om approximation only to attraction and keeps repulsion exact. This removes the main training-data bottleneck from the inner loop while preserving the stabilizing effect of exact generated-sample repulsion.

\subsection{Inference-time comparison}
\label{app:inference-time}

\begin{table*}[t]
\centering
\setlength{\tabcolsep}{4pt}
\caption{Inference-time benchmark. We report latency in milliseconds per generated sample from the final trained checkpoint. The inference times are similar as there both methods use the same one-step image generation.}
\label{tab:inference_meta_compare_3seed_model_only}
\begin{tabular}{llc}
\toprule
\textbf{Dataset} & \textbf{Method} & \textbf{Inference time (ms/sample)} $\downarrow$ \\
\midrule
CIFAR10  & Standard drifting & $0.177512 \pm 0.000062$ \\
CIFAR10  & DriftXpress        & $0.177348 \pm 0.000055$ \\
\midrule
CIFAR100 & Standard drifting & $0.177442 \pm 0.000124$ \\
CIFAR100 & DriftXpress        & $0.177417 \pm 0.000119$ \\
\midrule
SVHN      & Standard drifting & $0.177393 \pm 0.000004$ \\
SVHN      & DriftXpress        & $0.177426 \pm 0.000018$ \\
\midrule
ImageNet  & Standard drifting & $0.177264 \pm 0.000027$ \\
ImageNet  & DriftXpress        & $0.177285 \pm 0.000022$ \\
\bottomrule
\end{tabular}
\end{table*}

Table~\ref{tab:inference_meta_compare_3seed_model_only} compares the inference-time latency of standard drifting and DriftXpress. The two methods have essentially identical latency across all datasets, with all measurements around $0.177$ ms per generated sample. This is expected: DriftXpress changes the field-estimation procedure used during training, but it does not change the generator architecture or the one-step sampling procedure used at inference time.

This confirms the intended trade-off. DriftXpress reduces the wall-clock cost of training while preserving the inference-time advantage of drifting models. Once training is complete, both standard drifting and DriftXpress generate samples with the same single generator evaluation.

\begin{table}[t]
\centering
\small
\setlength{\tabcolsep}{5pt}
\caption{Per-iteration field-estimation profile on CIFAR10 using random per-class landmark selection. Entries report mean $\pm$ standard deviation over three seeds. Lower total field-estimation and full-iteration times are bolded.}
\label{tab:kernel_profile}
\resizebox{\textwidth}{!}{
\begin{tabular}{lccc}
\toprule
Component & Standard drifting & DriftXpress (sharded) & DriftXpress (unsharded) \\
\midrule
Exact attraction kernels $k(\mathbf x_b,\mathbf y_j^+)$
& $157.8 \pm 11.7$ ms & -- & -- \\

Nystr\"om query kernels $K(\mathbf x_b,\mathbf U)$
& -- & $40.1 \pm 0.1$ ms & $\mathbf{33.7 \pm 0.0}$ ms \\

Projected attraction $K_B\tilde{\mathbf A}_p,K_B\tilde{\mathbf b}_p$
& -- & $20.8 \pm 0.0$ ms & $\mathbf{15.4 \pm 0.0}$ ms \\

Remaining field reduction / repulsion / aggregation
& $111.6 \pm 0.1$ ms & $24.4 \pm 0.1$ ms & $\mathbf{24.0 \pm 0.1}$ ms \\

\midrule
Total field estimation
& $269.4 \pm 11.8$ ms & $85.3 \pm 0.2$ ms & $\mathbf{73.2 \pm 0.2}$ ms \\

Full training iteration
& $760.3 \pm 12.1$ ms & $576.6 \pm 0.9$ ms & $\mathbf{564.5 \pm 0.1}$ ms \\
\bottomrule
\end{tabular}
}
\end{table}

\begin{table}[t]
\centering
\small
\setlength{\tabcolsep}{5pt}
\caption{One-time DriftXpress preprocessing cost using random per-class landmark selection. Entries report mean $\pm$ standard deviation over three seeds. The largest component for each dataset is bolded.}
\label{tab:precompute_profile}
\resizebox{\textwidth}{!}{
\begin{tabular}{lccccc}
\toprule
Dataset & Feature extraction & Landmark selection & $\mathbf K_{UU}$ build + $\mathbf K_{UU}^{-1/2}$ & Summary construction & Total \\
\midrule
CIFAR10 & \textbf{19.7 $\pm$ 0.4 s} & 2.4 $\pm$ 0.1 s & 7.4 $\pm$ 0.1 s & 4.6 $\pm$ 0.0 s & 37.5 $\pm$ 0.5 s \\
CIFAR100 & 19.6 $\pm$ 0.2 s & 2.5 $\pm$ 0.0 s & \textbf{25.2 $\pm$ 0.2 s} & 6.6 $\pm$ 0.1 s & 59.6 $\pm$ 0.3 s \\
SVHN & \textbf{28.1 $\pm$ 2.6 s} & 3.7 $\pm$ 0.6 s & 19.1 $\pm$ 0.1 s & 9.9 $\pm$ 1.8 s & 63.0 $\pm$ 5.0 s \\
\bottomrule
\end{tabular}
}
\end{table}

\subsection{DriftXpress Cost Profile}
\label{app:cost-profile}

We profile the online and offline costs of DriftXpress separately. For the online profile, we use the CIFAR10 configuration matching the main benchmark setting. Landmarks are selected by random per-class sampling, and all times are averaged over three seeds. We report both class-sharded and unsharded DriftXpress. For the offline profile, we measure the one-time preprocessing pipeline on SVHN, CIFAR10, CIFAR100, and ImageNet.

\subsubsection{Online Cost Profile}
Table~\ref{tab:kernel_profile} reports the per-iteration profile. Standard drifting spends $269.4 \pm 11.8$ ms on field estimation. DriftXpress reduces this to $85.3 \pm 0.2$ ms in the sharded configuration and $73.2 \pm 0.2$ ms in the unsharded configuration, corresponding to reductions of $68.3\%$ and $72.8\%$, respectively. End-to-end iteration time decreases from $760.3 \pm 12.1$ ms for standard drifting to $576.6 \pm 0.9$ ms for sharded DriftXpress and $564.5 \pm 0.1$ ms for unsharded DriftXpress, corresponding to full-iteration reductions of $24.2\%$ and $25.8\%$.

The profile also identifies where the savings occur. Standard drifting spends $157.8 \pm 11.7$ ms on exact attraction kernels and $111.6 \pm 0.1$ ms on repulsion, field normalization, and aggregation. DriftXpress removes the exact attraction block and replaces it with Nystr\"om query kernels plus projected attraction. These two blocks cost $40.1 \pm 0.1$ ms and $20.8 \pm 0.0$ ms in the sharded variant, and $33.7 \pm 0.0$ ms and $15.4 \pm 0.0$ ms in the unsharded variant. The remaining repulsion and aggregation cost is $24.4 \pm 0.1$ ms for sharded DriftXpress and $24.0 \pm 0.1$ ms for unsharded DriftXpress. The unsharded variant is faster because it avoids per-shard projection and accumulation overhead. The sharded variant is still useful when the full summary does not fit in memory or when classwise partitioning is needed for larger settings.

Total field estimation measures only the cost of computing the drifting field: attraction, repulsion, normalization, and aggregation. Full training iteration measures the entire optimization step, including field estimation as well as the generator forward pass, feature computation, loss evaluation, backpropagation, and optimizer update.

\subsubsection{Computational steps in Table~\ref{tab:kernel_profile}}
Let $\{\mathbf x_b\}_{b=1}^{B}$ denote the generated query batch, let
$\{\mathbf y_j^+\}_{j=1}^{N^+}$ denote the positive reference samples used for attraction, and let
$\{\mathbf y_\ell^-\}_{\ell=1}^{N^-}$ denote the generated samples used for repulsion. All quantities are evaluated in the pre-trained feature representation space that serves as the input domain of the drifting kernel.

\paragraph{Exact attraction kernels $k(\mathbf x_b,\mathbf y_j^+)$.}
In standard drifting, the attractive component forms the pairwise kernel matrix
\[
  \mathbf K^+_{bj}
  =
  k(\mathbf x_b,\mathbf y_j^+),
  \qquad
  b=1,\dots,B,\quad j=1,\dots,N^+ .
\]
These weights define the exact attractive barycenter
\[
  \boldsymbol{\mu}_p(\mathbf x_b)
  =
  \frac{\sum_{j=1}^{N^+}\mathbf K^+_{bj}\mathbf y_j^+}
       {\sum_{j=1}^{N^+}\mathbf K^+_{bj}}.
\]
This row measures the time spent forming the exact attractive kernel weights between the generated batch and the positive reference samples. This is the repeated training-data interaction that \method{} removes from the online loop.

\paragraph{Nystr\"om query kernels $K(\mathbf x_b,\mathbf U)$.}
\method{} replaces the exact attractive kernel matrix with query-landmark kernels. Given landmarks
$\mathbf U=\{\mathbf u_1,\dots,\mathbf u_r\}$, it forms
\[
  \mathbf K_{BU}[b,a]
  =
  k(\mathbf x_b,\mathbf u_a),
  \qquad
  b=1,\dots,B,\quad a=1,\dots,r .
\]
With the precomputed transform
\[
  \mathbf W=(\mathbf K_{UU}+\lambda\mathbf I)^{-1/2},
\]
the query features are
\[
  \boldsymbol{\Phi}_B
  =
  \mathbf K_{BU}\mathbf W
  \in \mathbb R^{B\times r},
  \qquad
  [\boldsymbol{\Phi}_B]_{b,:}
  =
  \boldsymbol{\varphi}(\mathbf x_b)^\top .
\]
This row measures the online cost of computing the Nystr\"om features for the generated queries.

\paragraph{Projected attraction $K_B\tilde{\mathbf A}_p, K_B\tilde{\mathbf b}_p$.}
The positive summaries are precomputed offline as
\[
  \mathbf A_p
  =
  \sum_{j=1}^{N^+}
  \boldsymbol{\varphi}(\mathbf y_j^+)\mathbf y_j^{+\top}
  \in \mathbb R^{r\times D},
  \qquad
  \mathbf b_p
  =
  \sum_{j=1}^{N^+}
  \boldsymbol{\varphi}(\mathbf y_j^+)
  \in \mathbb R^r .
\]
Equivalently, after folding the Nystr\"om transform into the summaries, we denote the stored tensors by
$\tilde{\mathbf A}_p$ and $\tilde{\mathbf b}_p$. The projected attractive numerator and denominator for the full batch are computed as
\[
  \mathbf N_p^{\mathbf U}
  =
  \boldsymbol{\Phi}_B\mathbf A_p
  \in \mathbb R^{B\times D},
  \qquad
  \mathbf d_p^{\mathbf U}
  =
  \boldsymbol{\Phi}_B\mathbf b_p
  \in \mathbb R^B .
\]
For each query,
\[
  \boldsymbol{\mu}^{\mathbf U}_p(\mathbf x_b)
  =
  \frac{
    [\mathbf N_p^{\mathbf U}]_{b,:}
  }{
    [\mathbf d_p^{\mathbf U}]_b
  }.
\]
This row measures the low-rank matrix products used to recover the attractive barycenter from the precomputed summaries.

\paragraph{Remaining field reduction / repulsion / aggregation.}
This row collects the remaining online field computation after the attraction-specific block has been evaluated. For standard drifting, it includes the coupled dense exact-field reduction after forming exact attraction kernels, together with the repulsive computation and field aggregation. For \method{}, attraction has already been summarized by the projected blocks above, so this row consists of exact batch repulsion and lightweight aggregation.

The exact repulsive component forms the generated-sample kernel matrix
\[
  \mathbf K^-_{b\ell}
  =
  k(\mathbf x_b,\mathbf y_\ell^-),
  \qquad
  b=1,\dots,B,\quad \ell=1,\dots,N^- .
\]
In practice, when the query batch and repulsive support are the same generated batch, the diagonal self-interactions are masked before normalization. The exact repulsive barycenter is
\[
  \boldsymbol{\mu}_q(\mathbf x_b)
  =
  \frac{
    \sum_{\ell=1}^{N^-}\mathbf K^-_{b\ell}\mathbf y_\ell^-
  }{
    \sum_{\ell=1}^{N^-}\mathbf K^-_{b\ell}
  }.
\]
The final field is then assembled as
\[
  \mathbf V_{p,q}(\mathbf x_b)
  =
  \boldsymbol{\mu}_p(\mathbf x_b)
  -
  \boldsymbol{\mu}_q(\mathbf x_b)
\]
for standard drifting, and as
\[
  \mathbf V^{\mathrm{DX}}_{p,q}(\mathbf x_b)
  =
  \boldsymbol{\mu}^{\mathbf U}_p(\mathbf x_b)
  -
  \boldsymbol{\mu}_q(\mathbf x_b)
\]
for \method{}. When multiple temperatures, feature groups, or resolutions are used, this step also includes per-group normalization and aggregation, e.g.
\[
  \mathbf V(\mathbf x_b)
  =
  \sum_{g\in\mathcal G}
  \alpha_g
  \frac{\mathbf V_g(\mathbf x_b)}
       {\|\mathbf V_g(\mathbf x_b)\|+\varepsilon}.
\]
Thus, this row should not be interpreted as an isolated timing of the negative barycenter alone. It measures the remaining field-reduction, exact-repulsion, normalization, and aggregation work after the method-specific attraction computation has been accounted for.

\paragraph{Total field estimation.}
This is the total time spent computing the drifting field used to construct the training target. For standard drifting, it includes exact attraction and the remaining field-reduction, repulsion, and aggregation block:
\[
  T_{\mathrm{field}}^{\mathrm{std}}
  =
  T_{\mathrm{attr}}^{\mathrm{exact}}
  +
  T_{\mathrm{rem}}^{\mathrm{std}}.
\]
For \method{}, it includes Nystr\"om query kernels, projected attraction, and the remaining exact-repulsion and aggregation block:
\[
  T_{\mathrm{field}}^{\mathrm{DX}}
  =
  T_{\mathrm{query}\text{-}\mathrm{Nys}}
  +
  T_{\mathrm{attr}}^{\mathrm{proj}}
  +
  T_{\mathrm{rem}}^{\mathrm{DX}}.
\]
This quantity isolates the cost of estimating the drifting field itself: attraction, repulsion, normalization, and aggregation.

\paragraph{Full training iteration.}
This is the end-to-end time for one optimization step. It includes field estimation as well as all other training-loop components:
\[
  T_{\mathrm{iter}}
  =
  T_{\mathrm{encoder}}
  +
  T_{\mathrm{field}}
  +
  T_{\mathrm{loss/backward}}
  +
  T_{\mathrm{optimizer}}.
\]
Thus, total field estimation measures only the drifting-field computation, while full training iteration measures the complete optimization step, including encoder forward passes, construction of the drifted target, generator loss evaluation, backpropagation, and optimizer update. This is why a large reduction in field-estimation time translates into a smaller, but still substantial, reduction in end-to-end iteration time.

\subsubsection{Offline Cost‌ Profile}
Table~\ref{tab:precompute_profile} reports one-time preprocessing costs with random per-class landmarks. On CIFAR10, total preprocessing takes $37.5 \pm 0.5$ s. Feature extraction is the largest component at $19.7 \pm 0.4$ s, followed by the combined $\mathbf K_{UU}$ construction and $\mathbf K_{UU}^{-1/2}$ factorization stage at $7.4 \pm 0.1$ s. On CIFAR100, total preprocessing increases to $59.6 \pm 0.3$ s, with the $\mathbf K_{UU}$ stage becoming the largest component at $25.2 \pm 0.2$ s. On SVHN, preprocessing takes $63.0 \pm 5.0$ s, dominated by feature extraction at $28.1 \pm 2.6$ s and the $\mathbf K_{UU}$ stage at $19.1 \pm 0.1$ s. Random landmark selection is cheap across these datasets, taking only $2.4$--$3.7$ s.

\paragraph{Takeaway.}
DriftXpress moves the training-data interactions to a one-time preprocessing stage, after which attraction is evaluated through Nystr\"om query kernels and low-rank matrix products. With random per-class landmarks, preprocessing takes under one minute on CIFAR10 and CIFAR100 and about one minute on SVHN. The main offline costs are feature extraction and $\mathbf K_{UU}$ construction/factorization, while landmark selection is negligible.

These measurements are from a single-GPU profile. In multi-GPU training, the benefit of removing exact training-data interactions from the inner loop can be larger, since each device evaluates the drifting field for its local generated batch. Standard drifting repeats attraction against the training support on every device and every step, whereas DriftXpress reuses the same precomputed summaries and only evaluates query-landmark kernels plus low-rank products locally. Thus, the one-time preprocessing cost is amortized across devices and training steps, while the per-iteration savings scale with the number of generated batches evaluated during training.

\subsection{Field Approximation Fidelity}
\label{app:field-fidelity}

We next isolate the quality of the projected field approximation, independent of generator training.
On CIFAR10, we compare the exact attractive field $V_p^+(\mathbf x)$ with the projected attractive field $V_{p,U}^+(\mathbf x)$ induced by different landmark budgets.
We use random per-class landmark selection with $25$, $50$, $100$, $250$, $500$, $1000$, and $2000$ landmarks per class.
Since CIFAR10 has $5000$ training samples per class, these budgets correspond to $0.5\%$, $1\%$, $2\%$, $5\%$, $10\%$, $20\%$, and $40\%$ of each class support, respectively.
For each budget, we report three fidelity metrics. Cosine similarity measures directional agreement between the exact and projected drift vectors:
\[
    \frac{1}{B}\sum_{b=1}^B
    \frac{
        \langle V_p^+(\mathbf x_b), V_{p,U}^+(\mathbf x_b) \rangle
    }{
        \|V_p^+(\mathbf x_b)\|_2 \, \|V_{p,U}^+(\mathbf x_b)\|_2
    } .
\]
Relative $\ell_2$ error measures the size of the approximation error normalized by the size of the exact field:
\[
    \frac{
        \|V_{p,U}^+(\mathbf X)-V_p^+(\mathbf X)\|_F
    }{
        \|V_p^+(\mathbf X)\|_F
    } .
\]
Target MSE measures the mean squared error between the exact drift target and the projected drift target, where
$\mathbf t_b=\mathbf x_b+V_p^+(\mathbf x_b)$ and
$\mathbf t_{b,U}=\mathbf x_b+V_{p,U}^+(\mathbf x_b)$:
\[
    \frac{1}{BD}\sum_{b=1}^B
    \|\mathbf t_{b,U}-\mathbf t_b\|_2^2 .
\]
Thus, cosine similarity evaluates whether the projected field points in the right direction, relative $\ell_2$ error evaluates the normalized field-distortion magnitude, and target MSE evaluates the error in the actual regression target used by the drifting loss.

Table~\ref{tab:cifar10-field-fidelity} shows that increasing the landmark budget consistently improves the fidelity of the projected field.
With only $25$ landmarks per class, the projected field has cosine similarity $0.7426 \pm 0.0002$ and relative $\ell_2$ error $0.8042 \pm 0.0019$, indicating that the landmark basis is too small to accurately reproduce the exact field.
At $500$ landmarks per class, corresponding to $10\%$ of each class support, cosine similarity increases to $0.9683 \pm 0.0003$ and relative $\ell_2$ error drops to $0.3250 \pm 0.0067$.
At the largest budget, $2000$ landmarks per class, the projected field is very close to the exact field, reaching cosine similarity $0.9930 \pm 0.0007$, relative $\ell_2$ error $0.1758 \pm 0.0273$, and target MSE $0.0317 \pm 0.0096$.

This trend directly matches the approximation guarantees in Theorem~\ref{thm:local_field_distortion_main} and Corollary~\ref{cor:on_support_field_distortion_main}.
Theorem~\ref{thm:local_field_distortion_main} shows that the query-wise distortion of the projected attractive field is controlled by the landmark kernel residual $\|r_U(\mathbf x)\|_2$ normalized by the local kernel mass $d_p(\mathbf x)$.
Corollary~\ref{cor:on_support_field_distortion_main} gives the corresponding on-support statement: the average squared field distortion is controlled by the Gram approximation error $\|\mathbf K-\mathbf K_U\|_F^2$.
Increasing the number of landmarks improves the Nystr\"om approximation of the training-support kernel matrix, reducing $\|\mathbf K-\mathbf K_U\|_F^2$ and therefore reducing the field error predicted by the theory.
The empirical decrease in relative $\ell_2$ error and target MSE in Table~\ref{tab:cifar10-field-fidelity} is consistent with this prediction.
It also explains why \method{} can preserve sample quality in the main benchmark: once the landmark basis captures the dominant geometry of the training distribution, the projected attractive field closely matches the exact attractive field while being much cheaper to evaluate during training.

\begin{table}[t]
\centering
\small
\setlength{\tabcolsep}{6pt}
\caption{Field approximation fidelity on CIFAR10.}
\label{tab:cifar10-field-fidelity}
\resizebox{\textwidth}{!}{
\begin{tabular}{ccccc}
\toprule
Landmarks / class & Total landmarks & Cosine similarity $\uparrow$ & Relative $\ell_2$ error $\downarrow$ & Target MSE $\downarrow$ \\
\midrule
25   & 250   & $0.7426 \pm 0.0002$ & $0.8042 \pm 0.0019$ & $0.6468 \pm 0.0031$ \\
50   & 500   & $0.8166 \pm 0.0010$ & $0.6820 \pm 0.0043$ & $0.4651 \pm 0.0058$ \\
100  & 1000  & $0.8785 \pm 0.0006$ & $0.5674 \pm 0.0035$ & $0.3219 \pm 0.0039$ \\
250  & 2500  & $0.9408 \pm 0.0005$ & $0.4235 \pm 0.0039$ & $0.1794 \pm 0.0033$ \\
500  & 5000  & $0.9683 \pm 0.0003$ & $0.3250 \pm 0.0067$ & $0.1057 \pm 0.0044$ \\
1000 & 10000 & $0.9840 \pm 0.0007$ & $0.2456 \pm 0.0182$ & $0.0607 \pm 0.0088$ \\
2000 & 20000 & $\mathbf{0.9930 \pm 0.0007}$ & $\mathbf{0.1758 \pm 0.0273}$ & $\mathbf{0.0317 \pm 0.0096}$ \\
\bottomrule
\end{tabular}
}
\end{table}

Figure~\ref{fig:cifar10_class_comparison} compares class-conditional samples from \method{} and standard drifting on CIFAR10.  Figure~\ref{fig:cifar100_class_comparison} provides the same comparison for 10 randomly selected CIFAR100 classes. These qualitative results are intended to complement the FID scores in Section~\ref{sec:experiments}: while \method{} changes the estimator used during training, it does not change the generator architecture, the one-step sampling procedure, or the target image domain. All profiled models correspond to the same checkpoints reported in Table~\ref{tab:main_results_big}.

Across both datasets, the samples produced by \method{} are visually comparable to those produced by standard drifting. The class identity is generally preserved, and we do not observe an obvious qualitative degradation from replacing exact attraction with projected summaries. This supports the main empirical claim that the computational gains of \method{} do not come from a visibly weaker generator, but from reducing the cost of estimating the attractive component of the drifting field. The CIFAR100 comparison is particularly useful because the larger number of classes makes the attraction summaries more demanding; nevertheless, the projected field still produces samples that are qualitatively close to the standard drifting baseline.

\begin{figure}[t]
    \centering
    \includegraphics[width=0.9\linewidth]{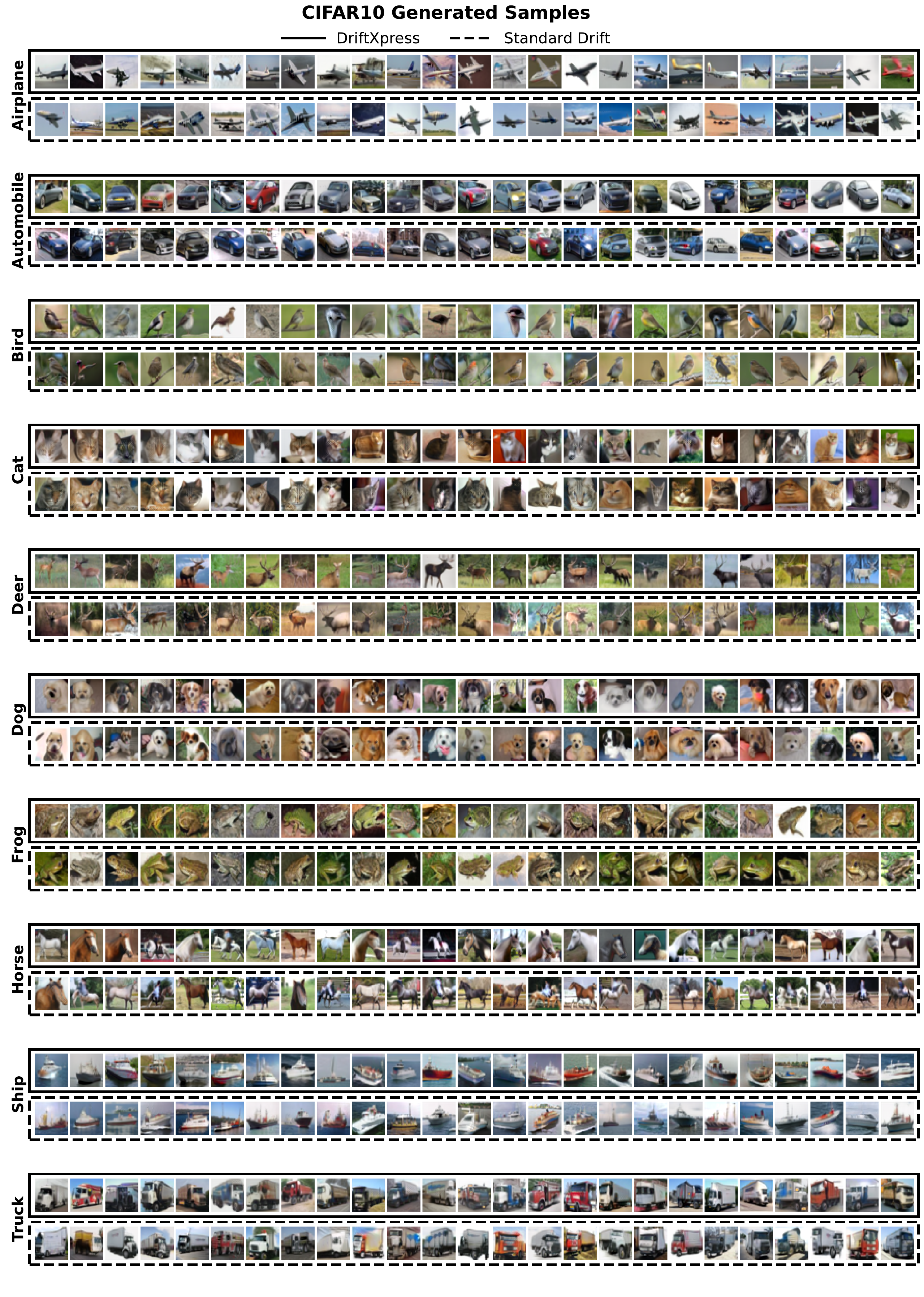}
    \caption{\textbf{CIFAR10 samples generated by DriftXpress and standard drifting.} For each class, the top row shows samples from DriftXpress and the bottom row shows samples from standard drifting. Both methods produce visually comparable samples across all CIFAR10 categories.}
    \label{fig:cifar10_class_comparison}
\end{figure}

\begin{figure}[t]
    \centering
    \includegraphics[width=0.9\linewidth]{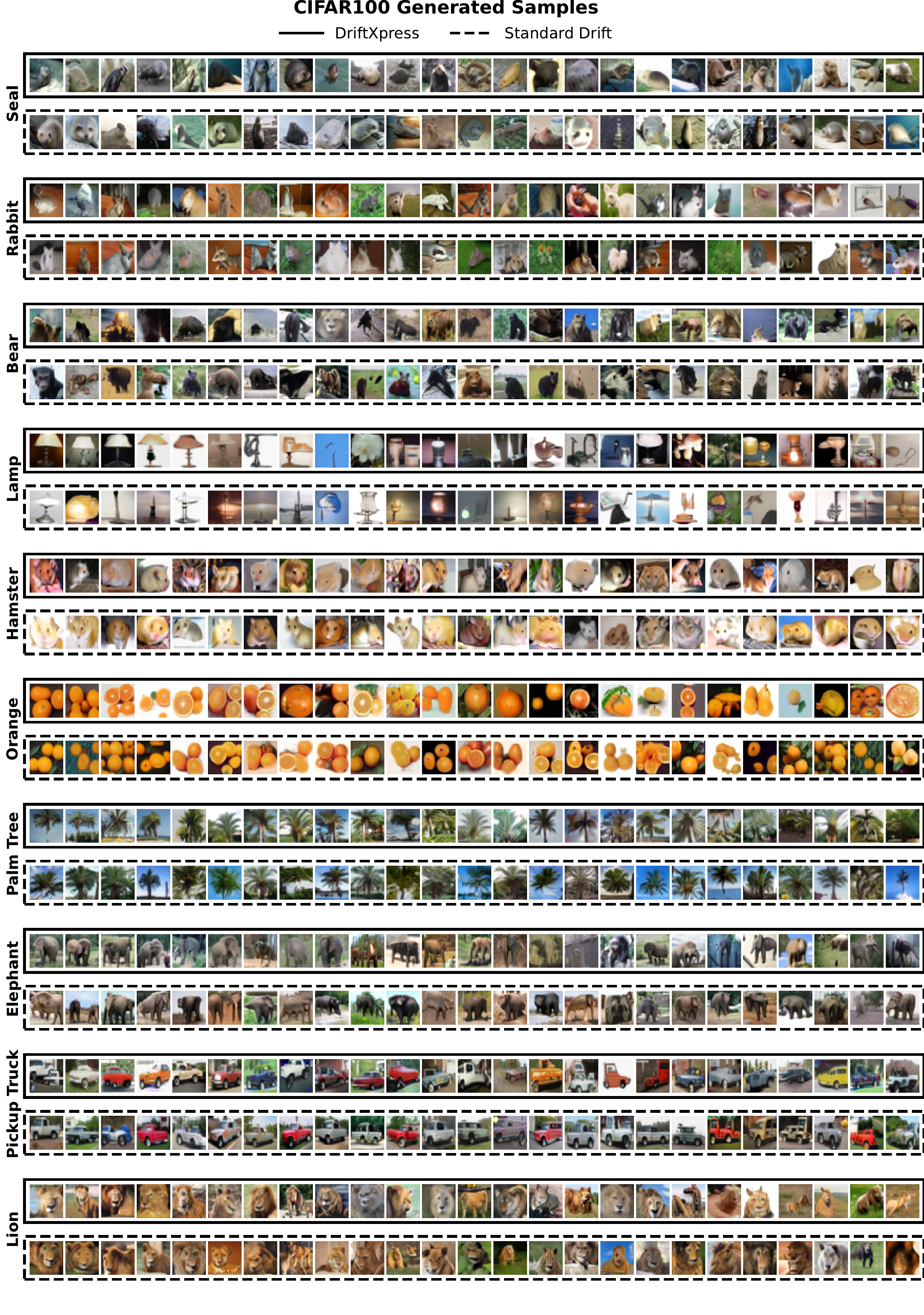}
\caption{\textbf{CIFAR100 samples generated by \method{} and standard drifting.}
We show 10 randomly selected CIFAR100 classes. For each class, the top row shows samples from \method{} and the bottom row shows samples from standard drifting. Both methods produce visually comparable samples across the selected CIFAR100 categories.}    \label{fig:cifar100_class_comparison}
\end{figure}

\section{Proofs}
\label{app:proofs}


\begin{theorem}[Compositionality of Nystr\"om attractive summaries]
Let
\[
\bar{\boldsymbol{\varphi}}(\mathbf x)
:=
\big[
\boldsymbol{\varphi}^{(1)}(\mathbf x);
\dots;
\boldsymbol{\varphi}^{(S)}(\mathbf x)
\big]
\]
denote the concatenation of the shard feature maps, and define the concatenated attractive summaries
\[
\bar{\mathbf A}_p
:=
\begin{bmatrix}
\mathbf A_p^{(1)}\\
\vdots\\
\mathbf A_p^{(S)}
\end{bmatrix},
\qquad
\bar{\mathbf b}_p
:=
\begin{bmatrix}
\mathbf b_p^{(1)}\\
\vdots\\
\mathbf b_p^{(S)}
\end{bmatrix}.
\]
Then the compositional attractive mean satisfies
\[
\mu_{p,\mathrm{comp}}(\mathbf x)
=
\frac{\bar{\boldsymbol{\varphi}}(\mathbf x)^\top \bar{\mathbf A}_p}
     {\bar{\boldsymbol{\varphi}}(\mathbf x)^\top \bar{\mathbf b}_p}.
\]
Consequently, when the same exact repulsive mean $\mu_q(\mathbf x)$ is used, the compositional field
\[
\mathbf V_{p,q}^{\mathrm{comp}}(\mathbf x)
:=
\mu_{p,\mathrm{comp}}(\mathbf x)-\mu_q(\mathbf x)
\]
is exactly equal to the field obtained from a single concatenated attractive summary and the same exact repulsive term.
\end{theorem}

\begin{proof}
By the definition of block concatenation,
\[
\bar{\boldsymbol{\varphi}}(\mathbf x)^\top \bar{\mathbf A}_p
=
\sum_{s=1}^S 
\boldsymbol{\varphi}^{(s)}(\mathbf x)^\top \mathbf A_p^{(s)},
\qquad
\bar{\boldsymbol{\varphi}}(\mathbf x)^\top \bar{\mathbf b}_p
=
\sum_{s=1}^S 
\boldsymbol{\varphi}^{(s)}(\mathbf x)^\top \mathbf b_p^{(s)}.
\]
Substituting these identities into the single-summary attractive mean gives
\[
\frac{\bar{\boldsymbol{\varphi}}(\mathbf x)^\top \bar{\mathbf A}_p}
     {\bar{\boldsymbol{\varphi}}(\mathbf x)^\top \bar{\mathbf b}_p}
=
\frac{
\sum_{s=1}^S 
\boldsymbol{\varphi}^{(s)}(\mathbf x)^\top \mathbf A_p^{(s)}
}{
\sum_{s=1}^S 
\boldsymbol{\varphi}^{(s)}(\mathbf x)^\top \mathbf b_p^{(s)}
}
=
\mu_{p,\mathrm{comp}}(\mathbf x).
\]
The repulsive mean $\mu_q(\mathbf x)$ is computed exactly over the current generated batch and is unchanged by the sharding of the positive summaries. Therefore,
\[
\mathbf V_{p,q}^{\mathrm{comp}}(\mathbf x)
=
\mu_{p,\mathrm{comp}}(\mathbf x)-\mu_q(\mathbf x)
\]
is exactly the field obtained by evaluating the concatenated attractive summary and subtracting the same exact repulsive mean.
\end{proof}

\begin{theorem}[Distortion of the projected attractive field]
\label{thm:deterministic_field_distortion}
Let $\{\mathbf y_j^+\}_{j=1}^{N^+}\subset\mathbb R^D$ satisfy
\[
  \|\mathbf y_j^+\| \le R,
  \qquad j=1,\dots,N^+,
\]
and let $k$ be a nonnegative kernel. Let $k_{\mathbf U}$ be any projected kernel
induced by a landmark set $\mathbf U$, for example the Nystr\"om kernel from
Section~\ref{sec:method}. For any query $\mathbf x\in\mathbb R^D$, define
\[
  \boldsymbol{\mu}_p(\mathbf x)
  :=
  \frac{\frac{1}{N^+}\sum_{j=1}^{N^+} k(\mathbf x,\mathbf y_j^+) \mathbf y_j^+}
       {\frac{1}{N^+}\sum_{j=1}^{N^+} k(\mathbf x,\mathbf y_j^+)},
  \qquad
  \boldsymbol{\mu}_p^{\mathbf U}(\mathbf x)
  :=
  \frac{\frac{1}{N^+}\sum_{j=1}^{N^+} k_{\mathbf U}(\mathbf x,\mathbf y_j^+) \mathbf y_j^+}
       {\frac{1}{N^+}\sum_{j=1}^{N^+} k_{\mathbf U}(\mathbf x,\mathbf y_j^+)},
\]
and
\[
  \mathbf V_p^+(\mathbf x) := \boldsymbol{\mu}_p(\mathbf x) - \mathbf x,
  \qquad
  \mathbf V_{p,\mathbf U}^+(\mathbf x) := \boldsymbol{\mu}_p^{\mathbf U}(\mathbf x) - \mathbf x.
\]

Define the query-wise kernel residual vector
\[
  \mathbf r_{\mathbf U}(\mathbf x)
  :=
  \bigl(
    k(\mathbf x,\mathbf y_1^+) - k_{\mathbf U}(\mathbf x,\mathbf y_1^+),
    \dots,
    k(\mathbf x,\mathbf y_{N^+}^+) - k_{\mathbf U}(\mathbf x,\mathbf y_{N^+}^+)
  \bigr)^\top
  \in \mathbb R^{N^+},
\]
and the exact positive normalizer
\[
  d_p(\mathbf x)
  :=
  \frac{1}{N^+}\sum_{j=1}^{N^+} k(\mathbf x,\mathbf y_j^+).
\]
Assume $d_p(\mathbf x)>0$. If
\begin{equation}
\label{eq:local_denominator_condition}
  \|\mathbf r_{\mathbf U}(\mathbf x)\|_2
  \le
  \frac{\sqrt{N^+}\,d_p(\mathbf x)}{2},
\end{equation}
then
\begin{equation}
\label{eq:local_field_distortion}
  \bigl\|
    \mathbf V_{p,\mathbf U}^+(\mathbf x) - \mathbf V_p^+(\mathbf x)
  \bigr\|
  \le
  \frac{4R}{\sqrt{N^+}\,d_p(\mathbf x)}
  \|\mathbf r_{\mathbf U}(\mathbf x)\|_2.
\end{equation}
Moreover, for the asymmetric \method{} field that uses projected attraction
and exact repulsion,
\[
  \mathbf V_{p,p}^{\mathbf U}(\mathbf x)
  =
  \boldsymbol{\mu}_p^{\mathbf U}(\mathbf x)-\boldsymbol{\mu}_p(\mathbf x),
\]
and therefore the same bound controls the equilibrium residual
$\|\mathbf V_{p,p}^{\mathbf U}(\mathbf x)\|$.
\end{theorem}

\begin{proof}
Fix a query $\mathbf x$. Define
\[
  \mathbf n_p(\mathbf x)
  :=
  \frac{1}{N^+}\sum_{j=1}^{N^+} k(\mathbf x,\mathbf y_j^+) \mathbf y_j^+,
  \qquad
  d_p(\mathbf x)
  :=
  \frac{1}{N^+}\sum_{j=1}^{N^+} k(\mathbf x,\mathbf y_j^+),
\]
and similarly
\[
  \mathbf n_{p,\mathbf U}(\mathbf x)
  :=
  \frac{1}{N^+}\sum_{j=1}^{N^+} k_{\mathbf U}(\mathbf x,\mathbf y_j^+) \mathbf y_j^+,
  \qquad
  d_{p,\mathbf U}(\mathbf x)
  :=
  \frac{1}{N^+}\sum_{j=1}^{N^+} k_{\mathbf U}(\mathbf x,\mathbf y_j^+).
\]
Then
\[
  \boldsymbol{\mu}_p(\mathbf x)
  =
  \frac{\mathbf n_p(\mathbf x)}{d_p(\mathbf x)},
  \qquad
  \boldsymbol{\mu}_p^{\mathbf U}(\mathbf x)
  =
  \frac{\mathbf n_{p,\mathbf U}(\mathbf x)}{d_{p,\mathbf U}(\mathbf x)}.
\]
Since the $\mathbf x$ terms cancel,
\[
  \mathbf V_{p,\mathbf U}^+(\mathbf x)-\mathbf V_p^+(\mathbf x)
  =
  \boldsymbol{\mu}_p^{\mathbf U}(\mathbf x)-\boldsymbol{\mu}_p(\mathbf x)
  =
  \frac{\mathbf n_{p,\mathbf U}(\mathbf x)}{d_{p,\mathbf U}(\mathbf x)}
  -
  \frac{\mathbf n_p(\mathbf x)}{d_p(\mathbf x)}.
\]

By definition of $\mathbf r_{\mathbf U}(\mathbf x)$,
\[
  \mathbf n_{p,\mathbf U}(\mathbf x)-\mathbf n_p(\mathbf x)
  =
  -\frac{1}{N^+}\sum_{j=1}^{N^+}
  [\mathbf r_{\mathbf U}(\mathbf x)]_j\,\mathbf y_j^+,
\]
and
\[
  d_{p,\mathbf U}(\mathbf x)-d_p(\mathbf x)
  =
  -\frac{1}{N^+}\sum_{j=1}^{N^+}
  [\mathbf r_{\mathbf U}(\mathbf x)]_j.
\]
Using $\|\mathbf y_j^+\|\le R$,
\[
  \|\mathbf n_{p,\mathbf U}(\mathbf x)-\mathbf n_p(\mathbf x)\|
  \le
  \frac{R}{N^+}\|\mathbf r_{\mathbf U}(\mathbf x)\|_1
  \le
  \frac{R}{\sqrt{N^+}}\|\mathbf r_{\mathbf U}(\mathbf x)\|_2,
\]
and
\[
  |d_{p,\mathbf U}(\mathbf x)-d_p(\mathbf x)|
  \le
  \frac{1}{N^+}\|\mathbf r_{\mathbf U}(\mathbf x)\|_1
  \le
  \frac{1}{\sqrt{N^+}}\|\mathbf r_{\mathbf U}(\mathbf x)\|_2.
\]

Now add and subtract $\mathbf n_p(\mathbf x)/d_{p,\mathbf U}(\mathbf x)$:
\[
  \left\|
    \frac{\mathbf n_{p,\mathbf U}(\mathbf x)}{d_{p,\mathbf U}(\mathbf x)}
    -
    \frac{\mathbf n_p(\mathbf x)}{d_p(\mathbf x)}
  \right\|
  \le
  \frac{
    \|\mathbf n_{p,\mathbf U}(\mathbf x)-\mathbf n_p(\mathbf x)\|
    +
    \left\|\frac{\mathbf n_p(\mathbf x)}{d_p(\mathbf x)}\right\|
    |d_{p,\mathbf U}(\mathbf x)-d_p(\mathbf x)|
  }{d_{p,\mathbf U}(\mathbf x)}.
\]
Since $k$ is nonnegative and $d_p(\mathbf x)>0$,
$\mathbf n_p(\mathbf x)/d_p(\mathbf x)$ is a convex combination of
$\{\mathbf y_j^+\}_{j=1}^{N^+}$. Therefore,
\[
  \left\|\frac{\mathbf n_p(\mathbf x)}{d_p(\mathbf x)}\right\|
  \le R.
\]
Moreover,
\[
  d_{p,\mathbf U}(\mathbf x)
  \ge
  d_p(\mathbf x)-|d_{p,\mathbf U}(\mathbf x)-d_p(\mathbf x)|
  \ge
  d_p(\mathbf x)
  -
  \frac{1}{\sqrt{N^+}}\|\mathbf r_{\mathbf U}(\mathbf x)\|_2.
\]
Under \eqref{eq:local_denominator_condition},
\[
  d_{p,\mathbf U}(\mathbf x)
  \ge
  \frac{d_p(\mathbf x)}{2}.
\]
Substituting these bounds gives
\[
  \left\|
    \frac{\mathbf n_{p,\mathbf U}(\mathbf x)}{d_{p,\mathbf U}(\mathbf x)}
    -
    \frac{\mathbf n_p(\mathbf x)}{d_p(\mathbf x)}
  \right\|
  \le
  \frac{
    \frac{R}{\sqrt{N^+}}\|\mathbf r_{\mathbf U}(\mathbf x)\|_2
    +
    R\frac{1}{\sqrt{N^+}}\|\mathbf r_{\mathbf U}(\mathbf x)\|_2
  }{d_p(\mathbf x)/2}
  =
  \frac{4R}{\sqrt{N^+}\,d_p(\mathbf x)}
  \|\mathbf r_{\mathbf U}(\mathbf x)\|_2.
\]
This proves \eqref{eq:local_field_distortion}.

For the final claim, when the model distribution matches the positive
distribution, the implemented asymmetric field uses projected attraction and
exact repulsion:
\[
  \mathbf V_{p,p}^{\mathbf U}(\mathbf x)
  =
  \boldsymbol{\mu}_p^{\mathbf U}(\mathbf x)
  -
  \boldsymbol{\mu}_p(\mathbf x).
\]
Thus,
\[
  \|\mathbf V_{p,p}^{\mathbf U}(\mathbf x)\|
  =
  \|\boldsymbol{\mu}_p^{\mathbf U}(\mathbf x)-\boldsymbol{\mu}_p(\mathbf x)\|,
\]
which is exactly the quantity bounded above.
\end{proof}

\begin{corollary}[On-support distortion]
\label{cor:on_support_field_distortion}
Let $\mathbf K,\mathbf K_{\mathbf U}\in\mathbb R^{N^+\times N^+}$ be the exact and projected
Gram matrices on $\{\mathbf y_j^+\}_{j=1}^{N^+}$, with entries
\[
  [\mathbf K]_{ij}=k(\mathbf y_i^+,\mathbf y_j^+),
  \qquad
  [\mathbf K_{\mathbf U}]_{ij}=k_{\mathbf U}(\mathbf y_i^+,\mathbf y_j^+).
\]
Define
\[
  \kappa_{\min}
  :=
  \min_{1\le i\le N^+}
  \frac{1}{N^+}\sum_{j=1}^{N^+} k(\mathbf y_i^+,\mathbf y_j^+),
\]
and assume
\[
  \|\mathbf K-\mathbf K_{\mathbf U}\|_{2,\infty}
  :=
  \max_{1\le i\le N^+}
  \|(\mathbf K-\mathbf K_{\mathbf U})_{i,:}\|_2
  \le
  \frac{\sqrt{N^+}\,\kappa_{\min}}{2}.
\]
Then
\[
  \frac{1}{N^+}\sum_{i=1}^{N^+}
  \bigl\|
    \mathbf V_{p,\mathbf U}^+(\mathbf y_i^+)
    -
    \mathbf V_p^+(\mathbf y_i^+)
  \bigr\|^2
  \le
  \frac{16R^2}{{N^+}^2\kappa_{\min}^2}
  \|\mathbf K-\mathbf K_{\mathbf U}\|_F^2.
\]
\end{corollary}

\begin{proof}
Apply Theorem~\ref{thm:deterministic_field_distortion} with $\mathbf x=\mathbf y_i^+$.
Then the residual vector $\mathbf r(\mathbf y_i^+)$ is exactly the $i$th row of
$\mathbf K-\mathbf K_{\mathbf U}$, and $d(\mathbf y_i^+) \ge \kappa_{\min}$.
Thus, for every $i$,
\[
  \bigl\|
    \mathbf V_{p,\mathbf U}^+(\mathbf y_i^+)
    -
    \mathbf V_p^+(\mathbf y_i^+)
  \bigr\|
  \le
  \frac{4R}{\sqrt{N^+}\,\kappa_{\min}}
  \|(\mathbf K-\mathbf K_{\mathbf U})_{i,:}\|_2.
\]
Square and average over $i$, then use
\[
  \sum_{i=1}^{N^+}\|(\mathbf K-\mathbf K_{\mathbf U})_{i,:}\|_2^2
  =
  \|\mathbf K-\mathbf K_{\mathbf U}\|_F^2.
\]
\end{proof}

\end{document}